\documentclass[aps,twocolumn,superscriptaddress,longbibliography,10pt,pre]{revtex4-2}

\usepackage{amsmath,amsfonts,bm}

\def\eqref#1{equation~\ref{#1}}

\def\1{\bm{1}}

\def\rmB{{\mathbf{B}}}

\DeclareMathAlphabet{\mathsfit}{\encodingdefault}{\sfdefault}{m}{sl}
\SetMathAlphabet{\mathsfit}{bold}{\encodingdefault}{\sfdefault}{bx}{n}

\def\gB{{\mathcal{B}}}
\def\gC{{\mathcal{C}}}

\def\gG{{\mathcal{G}}}

\def\gM{{\mathcal{M}}}

\def\gS{{\mathcal{S}}}
\def\gT{{\mathcal{T}}}

\def\sR{{\mathbb{R}}}

\def\sZ{{\mathbb{Z}}}

\DeclareMathOperator*{\argmin}{arg\,min}

\usepackage{xcolor}
\definecolor{darkblue}{rgb}{0,0,0.6}
\usepackage[colorlinks,linkcolor=darkblue,citecolor=darkblue,urlcolor=darkblue]{hyperref}
\usepackage{amsmath}
\usepackage{amsfonts}
\usepackage{amsthm}
\usepackage{verbatim}
\usepackage{graphicx}
\usepackage{esint}
\usepackage{amssymb}
\usepackage{mathrsfs}
\usepackage{mathtools}
\usepackage{bm}
\usepackage{bbm}
\usepackage{circledsteps}
\usepackage{tikz,standalone}
\usetikzlibrary{matrix,positioning}
\usepackage{physics}
\usepackage{mathtools}
\newcommand{\beq}{\begin{equation}}
\newcommand{\eeq}{\end{equation}}

\usepackage{soul}
\usepackage{url}
\usepackage{cancel}
\usepackage{enumitem}
\usepackage{multirow}
\usepackage{comment}
\usepackage{booktabs}
\usepackage[capitalize,noabbrev]{cleveref}
\usepackage[english]{babel}
\babeladjust{autoload.bcp47 = on}
\DeclareMathOperator{\atantwo}{atan2}
\def\eqsp{\;}

\newcommand{\iinter}[2]{[\![#1, #2]\!]}

\newcommand{\tcmb}[1]{{\color{black} #1}}
\newcommand{\selectparticle}[2]{#1_{(#2)}}
\newcommand{\tangenttorus}[1]{\mathcal{T}_{#1} \torus}
\newcommand{\torusdist}[2]{\mathrm{d}_{\torus}(#1,#2)}

\newcommand{\torusmod}[1]{\left(#1\right) ~\%~ L}

\newcommand{\torus}{\gM}
\newcommand{\invgroup}{G_{\gC}}
\newcommand{\pstar}{p_{\star}}
\newcommand{\pbase}{p_{\text{base}}}
\newcommand{\rmd}{\mathrm{d}}

\theoremstyle{plain}
\newtheorem{theorem}{Theorem}[section]
\newtheorem{proposition}[theorem]{Proposition}
\newtheorem{lemma}[theorem]{Lemma}
\newtheorem{corollary}[theorem]{Corollary}
\theoremstyle{definition}
\newtheorem{definition}[theorem]{Definition}

\theoremstyle{remark}

\begin{document}

\title{Boltzmann generators for amorphous particle systems}

\author{Louis Grenioux}

\thanks{These authors contributed equally.}

\affiliation{CMAP, CNRS, École polytechnique, Institut Polytechnique de Paris, 91120 Palaiseau, France}

\affiliation{Center for Computational Mathematics, Flatiron Institute, New York, NY, USA}

\author{Leonardo Galliano}

\thanks{These authors contributed equally.}

\affiliation{Laboratoire de Physique de l’École normale supérieure ENS, Université PSL, CNRS,  Sorbonne Université, Université de Paris, 75005 Paris, France}

\affiliation{Dipartimento di Fisica, Universit\`a di Trieste, Strada Costiera 11, 34151, Trieste, Italy}

\author{Ludovic Berthier}
\affiliation{Gulliver, CNRS UMR 7083, ESPCI Paris, PSL Research University, 75005 Paris, France}

\author{Giulio Biroli}
\affiliation{Laboratoire de Physique de l’École normale supérieure ENS, Université PSL, CNRS, Sorbonne Université, Université de Paris, 75005 Paris, France}

\author{Marylou Gabrié}
\affiliation{Laboratoire de Physique de l’École normale supérieure ENS, Université PSL, CNRS, Sorbonne Université, Université de Paris, 75005 Paris, France}

\date{\today}

\begin{abstract}
Sampling configurations in thermodynamic equilibrium is a long-standing challenge in statistical physics. Boltzmann generators address this problem by employing generative models to propose independent configurations, which are then reweighted via importance sampling using exact likelihood evaluations. Recent Boltzmann Generators based on continuous normalizing flows and flow matching have achieved significant success for particle systems and biomolecules. However, these approaches have not been extended to amorphous materials (glasses), for which equilibrium sampling is notoriously slow. Because of their disordered structure, the invariances and geometrical constraints of amorphous materials differ from those of crystals and biomolecules, preventing the direct use of existing generative models. Here, we develop Boltzmann Generators tailored to amorphous materials by building the required equivariances directly into Riemannian stochastic interpolants. Our framework incorporates periodic boundary conditions and particle symmetries using equivariant graph neural networks. Numerical experiments demonstrate that enforcing physical symmetries significantly improves the accuracy of Boltzmann Generators, but also reveal an intrinsic limitation of the continuous-flow formulation: accumulated numerical errors during likelihood integration break time-reversibility, compromising exact thermodynamic reweighting. These results reveal a fundamental challenge for continuous-flow generative models in statistical mechanics and call for alternative approaches that preserve exact thermodynamic consistency. 
\end{abstract}

\maketitle

\section{Introduction}

Sampling configurations at thermodynamic equilibrium is a long-standing challenge in statistical physics. The ability to generate equilibrium samples is essential for progress across a wide range of scientific applications, from biomolecular function and drug discovery to materials design \cite{liu2001monte, stoltz2010free, frenkel2023understanding,ohno2018computational}. In particular, accurate equilibrium sampling underlies the modeling of protein folding and binding, as well as crystal nucleation and growth \cite{parrinello1980crystal, matsumoto2002molecular, noe2009constructing,lindorff2011fast,buch2011complete}, with implications spanning global health and energy-storage technologies \cite{deringer2020modelling}. Consequently, developing methods for accurate and unbiased Boltzmann sampling remains a central objective for applications in chemistry, biology, and materials science.

The difficulty of Boltzmann sampling stems from the structure of the energy landscapes that define the equilibrium measure. For many systems of interest, the effective energy is high-dimensional, rugged, and populated by numerous metastable states sometimes separated by large free energy barriers. As a consequence, classical simulation-based methods such as Molecular Dynamics (MD) and Markov chain Monte Carlo (MCMC) can become trapped in local minima and require a computationally prohibitive number of steps to mix between relevant modes, thus producing strongly correlated samples and large statistical inefficiencies.

A recent line of work proposes to address this challenge by sampling through a trainable generator with tractable likelihoods, enabling importance-sampling reweighting to the target energy and unbiased estimation of observables. Early advances leveraged normalizing flows \citep{noe2019boltzmann, albergoFlowbasedGenerativeModels2019} and autoregressive models \citep{wuSolvingStatisticalMechanics2019}, and subsequent work has explored more expressive ODE-based continuous normalizing flows (CNFs) \cite{kohler2020equivariant, klein2023equivariant, rehman2025falconfewstepaccuratelikelihoods, schebek2026assessing}. Overall, this paradigm, sometimes refered to as Boltzmann generators (BGs) following \cite{noe2019boltzmann}, has yielded promising results for molecular systems and many-particle benchmarks.

Among the broad class of particle systems, amorphous materials remain a particularly challenging frontier. Amorphous particle systems are disordered arrangements of interacting particles (possibly with multiple species) that lack the atomic periodicity of crystals. Prominent examples include structural glasses and glass-forming liquids \citep{berthier2011theoretical}. A major obstacle to their theoretical understanding is the extremely long equilibration timescales, even at moderate system sizes. This severely limits the reach of standard numerical sampling approaches \citep{berthier2023modern}. Swap Monte Carlo extends the accessible regime, but remains restricted to specific models \citep{ninarello2017models}. These limitations make amorphous materials a natural and important target for equilibrium samplers based on learnable generators.

A critical requirement for extending Boltzmann generators to amorphous materials is to explicitly encode the geometry and invariances of the underlying physical system into the models. Prior work has developed efficient equivariant BGs \citep{kohler2020equivariant, klein2023equivariant}, but these models were not designed to capture the specific invariances and geometric constraints arising from periodic boundary conditions (PBCs) used to simulate amorphous systems. More recently, equivariant CNFs trained by maximum likelihood have been applied to amorphous materials \citep{Jung2024normalizingflows}, but at a prohibitive computational cost due to the necessity to estimate their likelihoods in training, as will be further discussed below. Diffusion-based models provide a cheaper, alternative and can generate visually realistic amorphous samples \citep{yang2025generative}. However, their lack of a tractable likelihood prevents principled reweighting and therefore precludes exact Boltzmann sampling.

Motivated by these limitations, we focus on stochastic interpolants (SIs), a generalization of flow matchings (FMs) \citep{albergo2023building,albergo2023stochastic,lipman2023flow}. These methods admit a cheap training objective while retaining an expressive CNF generation mechanism, allowing (i) straightforward incorporation of symmetries via equivariant velocity fields \citep{kohler2020equivariant,satorras2021en_nf} and (ii) tractable likelihoods that enable importance weight corrections and unbiased estimation of expectations of observables \citep{chen2019neural}. Moreover, recent extensions to Riemannian manifolds provide a principled route to handle periodic boundary conditions by working directly with the appropriate geometry \citep{chen2024flow,wu2025riemannian}, as already exploited for crystalline settings \citep{miller2024flowmm}. Building on these ideas, our goal is to tailor Boltzmann generators to amorphous materials by enforcing the relevant geometric structure and invariances by design.
Guided by these goals, we make the following contributions:
\vspace{-0.6em}
\begin{itemize}[leftmargin=5mm]
	\item We introduce the equivariant Riemannian stochastic interpolant (ERSI) framework, which extends Riemannian stochastic interpolants by incorporating invariance constraints tailored to amorphous particle systems.
	\item We prove that the optimal marginal paths and velocity fields induced by our objective respect the full symmetry group relevant to multi-species amorphous systems on the torus, providing novel theoretical guarantees within the stochastic interpolant framework.
	\item We adapt the graph neural network architecture of \cite{satorras2021en_gnn} to respect the full set of symmetries relevant for amorphous materials.
	\item We apply our framework to two canonical models of multi-component metallic glass formers and benchmark against symmetry- and geometry-agnostic baselines. Our approach yields higher-quality generations and improved agreement of physical observables after importance-sampling reweighting, while exhibiting improved scalability with system size.
    \item However, our study also uncovers a fundamental limitation to the use of CNF generators for reweighting due to the discretization errors arising in computing the models likelihood, raising doubts about its general applicability for applications in statistical physics.
\end{itemize}

Incorporating invariances into push-forward generative models has been studied in \cite{kohler2019equivariantflowssamplingconfigurations, kohler2020equivariant, Wirnsberger2022atomistic, Midgley2023se3, Bilos2024scalable}, with \cite{klein2023equivariant} adapting FM to build equivariant transport maps, deriving results akin to our theoretical results
{, but limited to linear actions on $\sR^n$}. In parallel, Denoising Diffusion Probabilistic Models (DDPMs) were extended with equivariances \citep{xu2022geodiff, Hoogeboom2022equivariant}, enabling generation but not likelihood evaluation. All these advances relied on equivariant GNNs \citep{satorras2021en_gnn, satorras2021en_nf}, which also underpin our approach. Interestingly, later works showed that strong generative performance could be obtained even without explicit invariance constraints \citep{martinkus2023abdiffuser, Chu2024anallatom, joshi2025allatom}, suggesting that symmetries enhance efficiency and generalization, though they may not be strictly required for generation. Our experiments below demonstrate that incorporating invariances is advantageous.

Generative modeling for materials has already attracted significant attention, in particular for crystalline structures. Like the systems considered here, crystals exhibit a number of symmetries and non-Euclidean representations. Most approaches build on Riemannian extensions of DDPMs \citep{yang2023scalable, jiao2023crystal, jiao2024space, Zeni2025agenerative, levy2025symmcd}, while works such as \cite{miller2024flowmm, Sriram2024flowllm} resonate more closely with our setting by embedding invariances into Riemannian extensions of FMs. In this body of work, success is typically assessed by visual or structural fidelity to the training configurations, which makes the task substantially easier than checking the fidelity to the equilibrium Boltzmann distribution. Moreover, unlike amorphous materials, crystals can be described by unit cells and fractional coordinates, due to their atomic periodicity, which entails qualitatively distinct invariance structures.

The paper is organized as follows. In \cref{sec:preliminaries} we recall the characteristics of particle mixtures and provide some background on generative models for sampling particle systems. In \cref{sec:ersi}, we present our equivariant Riemannian stochastic interpolant framework. In \cref{sec:methods}, we present the detailed setup of our numerical evaluation of ERSI, before presenting the results of our investigation in \cref{sec:numerical-results}.

\section{Generative models for multi-component particle systems}

\label{sec:preliminaries}

\subsection{Modeling mixtures of particles}

\subsubsection{Definitions and equilibrium distribution}

We consider classical models of multi-component particle systems, which describe a system of interacting particles at a constant temperature and volume. The microscopic state of the system is fully specified by a configuration $C = (s, X)$, where $s$ encodes the discrete chemical species of each particle from a finite set $\gS$, and $X$ denotes their spatial coordinates.

To minimize finite-size effects and accurately mimic bulk behavior, simulations must employ periodic boundary conditions. Mathematically, this corresponds to embedding the particle coordinates on a $d$-dimensional flat torus $\torus = [0,L)^d$, where $L$ is the linear size of the simulation box.
This defines a Riemannian manifold, where particles exiting one face of the box smoothly re-enter from the opposite side. The natural metric on $\torus$ is the nearest image distance
\begin{equation}\label{eq:torus:dist}
    \torusdist{\selectparticle{X}{i}}{\selectparticle{X}{j}} = \min_{k \in \sZ^d} \norm{\selectparticle{X}{i}-\selectparticle{X}{j} + kL}\eqsp,
\end{equation}
where $\selectparticle{X}{i} \in \torus$ denotes the coordinates of the $i$-th particle. 
Note that $\mathrm{d}_{\torus}$ is a distance on $\torus$ but not on $\sR^d$ since it does not satisfy the triangular inequality.
The full state space of the system is the product manifold $\gC = \gS^N\times\torus^N$.
On the torus, particle coordinates are not unique: for any $(s, X) \in \gC$, all configurations of the form $(s', X')$ with $s' = s, \; X' = X + kL, \;\; k \in \sZ^{Nd}$ represent the same physical state. This ambiguity is lifted by applying component-wise the modulo operator
$$
    A \;\% \; L = A - \left\lfloor \frac{A}{L} \right\rfloor L\eqsp, \qquad A \in \sR\eqsp,
$$
mapping all equivalent coordinates back into the fundamental domain. 

The thermodynamic equilibrium of the system is determined by a potential energy function $U:\mathcal C\to\mathbb R$. For typical particle models, this is expressed as a sum of pairwise interactions based on inter-particle distances,
\begin{align}\label{eq:potential}
    U(s, X) \;=\; \sum_{i=1}^N \sum_{j<i}^N \mathrm{W}\left(s_i, s_j, \torusdist{\selectparticle{X}{i}}{\selectparticle{X}{j}}\right)\eqsp,
\end{align}
where $\mathrm{W} : \gS \times \gS \times \sR^+ \to \sR$ is the interaction potential. At temperature $T$, the {equilibrium distribution} is the Boltzmann measure
\begin{equation}\label{eq:boltzmann}
    \pstar(\rmd s, \rmd \text{vol}_X) \;=\; \frac{1}{\mathcal{Z}}\exp\left(-\frac{U(s,X)}{k_{\rm B}T}\right) \, \rmd s \, \rmd \text{vol}_X\eqsp,
\end{equation}
where $k_{\rm B}$ is the Boltzmann constant, $\rmd \text{vol}_X$ is the volume element over $\torus^N$ and the partition function ${\mathcal{Z} = \int \exp\left(-U(s,X) / k_{\rm B}T\right) \rmd s\,\rmd \text{vol}_X}$ ensures the normalization. 
This joint distribution allows species to vary freely and correspond to sampling from a ``semi-grand" canonical ensemble, where the total number of particles is fixed but the composition is not.  Since standard models of glass-forming mixtures are defined at a fixed composition, we will also be interested in sampling from $\pstar$ conditioned on a composition, where the number of particles per species is strictly conserved. Sampling from this conditioned distribution corresponds to sampling the canonical ensemble. Our objective is to build a generative model that approximates $\pstar$ (or its conditional on the composition) to allow efficient sampling of equilibrium configurations.

\subsubsection{Symmetries}

\label{sec:invariances}

\begin{figure}[t!]
    \centering
    \includegraphics[width=\linewidth]{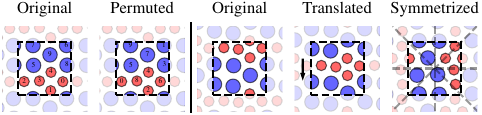}
    \caption{Illustration of invariance group actions on a particle configuration. The system contains two particle species with different effective diameters.}
    \label{fig:invariances}
\end{figure}

For a particle mixture at equilibrium, the potential energy, and consequently the Boltzmann distribution, are invariant under transformations that preserve relative distances and species composition. In a periodic box, these symmetries correspond to a specific set of transformations on the configuration space $\mathcal C$, denoted collectively by the group $\invgroup$ (illustrated in \Cref{fig:invariances}). This group comprises:
  
\begin{itemize}[leftmargin=5mm]
    \item \textbf{Particle permutations}, exchanging the indices of any two identical particles.
    For any permutation $\sigma \in S_N$, this transformation swaps both the species assignment and the spatial coordinates of the particles:
    $$
        g_{\sigma}(s, X)= \begin{pmatrix}
            s_{\sigma(1)}, \ldots, s_{\sigma(N)} \\
            \selectparticle{X}{\sigma(1)}, \ldots, \selectparticle{X}{\sigma(N)}
        \end{pmatrix}\eqsp.
    $$
    \item \textbf{Translations}, shifting all particle coordinates by the same vector and wrapping them back onto $\torus$. Denoting by $\mathbf{1}_N \in \sR^N$ the vector with all coordinates equal to $1$, for any $u \in \sR^d$:
    $$
        g_u(s, X) = \begin{pmatrix}
            s \\
            \torusmod{X + \mathbf{1}_N \otimes u}
        \end{pmatrix}\eqsp,
    $$
    \item \textbf{Box symmetries}, combining permutations of axes and sign flips of coordinates along these axes.
    Mathematically, this corresponds to multiplying the coordinates by a signed permutation matrix $M$ in the $d$-dimensional hyperoctahedral group $B_d$:
    $$
        g_M(s, X) = \begin{pmatrix}
            s\\
            \torusmod{(\mathrm{I}_N \otimes M) X}
        \end{pmatrix}\eqsp,
    $$
    where $\mathrm{I}_N$ is the identity matrix of size $N \times N$. 
\end{itemize}
\vspace{-0.6em}
While isolated particle systems such as single molecules may exhibit full continuous rotational invariance, restricting coordinates to a flat torus rigorously reduces these symmetries to signed permutations of the coordinate axes. 
We prove in \Cref{app:invariances} that $\pstar$ defined in~\Cref{eq:boltzmann} and its conditional on the composition are $\invgroup$-invariant, which means that applying any action $g\in\invgroup$ to a configuration leaves its probability density unchanged.

\subsection{Boltzmann generators}

We refer to Boltzmann generators for the set of enhanced sampling schemes designed to produce unbiased equilibrium configurations using normalizing flows \cite{noe2019boltzmann,albergoFlowbasedGenerativeModels2019,nicoliAsymptoticallyUnbiasedEstimation2020,nicoliEstimationThermodynamicObservables2021, gabrieAdaptiveMonteCarlo2022, kohler2019equivariantflowssamplingconfigurations}. 
Normalizing flows are generative models that consist in learning a deterministic invertible transport map that pushes a simple, tractable base distribution $\pbase$ (such as an ideal gas) towards a proposal distribution $\hat q$ that closely approximates a complex target $\pi$ (such as the equilibrium Boltzmann distribution). Thanks to the invertibility of the transport map, the likelihood  $\hat q$ of the generated samples can be evaluated exactly through a change of variable formula (see the continuous version below) and the proposal distribution can be formally corrected to the true target measure via importance sampling. Namely, denoting by $X \in \mathcal{X}$ the configurations in this generic setup, the expectation value of an observable $\phi(X)$ under the target distribution $\pi$ is estimated using the self-normalized importance sampling estimator:
\begin{equation}
\label{eq:snis}
    \overline{\phi}_{\pi}^R=\sum_{i=1}^R\phi(X_i)\,\frac{\overline w(X_i)}{\sum_{j=1}^R\overline w(X_j)}\eqsp,\quad X_i\sim\hat q~~\text{i.i.d.}\eqsp,
\end{equation}
where $\overline w(X_i)=\overline \pi(X_i)/\hat q(X_i)$ is the unnormalized importance weight of the $i$-the configuration requiring the knowledge of the unnormalized density $\overline \pi$ of $\pi$. 
This estimator implicitly assumes that the support of the target distribution is completely contained within the support of the model proposal.
While the self-normalized estimator $\overline{\phi}_{\pi}^R$ introduces a finite-sample bias, it remains asymptotically consistent and converges to the exact target expectation as the number of proposal samples $R \to \infty$.
\tcmb{Beyond computing expectation values, the unnormalized importance weights $\overline w$ directly yield estimates of free energy differences between the base and target systems~\cite{stoltz2010free,schebek2026assessing}, making the framework naturally suited for targeted free energy perturbation applications.}

A highly expressive framework for parameterizing these transport maps is Continuous Normalizing Flows (CNFs)  \cite{chen2019neural, grathwohl2018scalable}. In a CNF, the transformation is defined by an ordinary differential equation (ODE),
\begin{equation}\label{eq:neuralode}
    \rmd \hat{X}_t = \hat v(t, \hat{X}_t)\rmd t, \quad \hat{X}_0 \sim \pbase\eqsp,
\end{equation}
where $\hat v$ is a learned time-dependent velocity field for $t\in[0,1]$ and the output of the map is $\hat{X}_1$, the configuration reached at the final time. 

This formulation yields a tractable likelihood by integration of the instantaneous change-of-variables formula \citep[Theorem 1]{chen2019neural}. Specifically, the logarithm of the density $\hat q_t$ of $\hat X_t$ evolves along the ODE
\begin{align}\label{eq:ode_likelihood}
    \frac{{\rmd}}{{\rmd} t} \log \hat q_t(\hat X_t) \;=\; -\operatorname{div}\hat v(t,\hat X_t), \\
    \quad \log \hat q_0(\hat X_0) = \log \pbase(\hat X_0) \notag\eqsp,
\end{align}
from which $\hat{q}=\hat{q}_1$ can be computed.

While access to this exact likelihood theoretically allows the model to be trained directly via maximum likelihood estimation, doing so is computationally prohibitive for high-dimensional many-body systems. The bottleneck is twofold. First, evaluating the divergence term  in \cref{eq:ode_likelihood} requires computing the trace of the velocity field's Jacobian, demanding expensive auto-differentiation at every step (although this can be somewhat mitigated by restricting the model to specific, less expressive architectures \citep{kohler2019equivariantflowssamplingconfigurations}). Second, maximum-likelihood training is fundamentally not ``simulation-free'' as it requires numerically integrating the ODE and tracking its gradients at every single optimization step. Whether this is achieved through memory-efficient techniques like the adjoint method \citep{chen2019neural} or by discretizing the ODE prior to optimization \citep{Gholaminejad2019anode}, the fundamental requirement of repeatedly simulating the flow trajectories renders traditional maximum likelihood training of CNFs exceptionally slow. The recently introduced stochastic interpolant framework offers an alternative, as discussed next.

\subsection{Stochastic interpolants}
\label{sec:si}
{Stochastic interpolants} (SI) \citep{albergo2023building, albergo2023stochastic} are a generative modeling framework generalizing {flow matching} (FM) \citep{liu2023flow, lipman2023flow} designed to overcome the limitations of training CNFs via maximum likelihood. SI relies on an interpolation process $(X_t)_{t \in [0,1]}$ between the base distribution ($X_0 \sim \pbase$) and the target distribution ($X_1 \sim \pstar$). This process is defined through an interpolation function $X_t = I(t, X_0, X_1)$ satisfying the boundary conditions $I(0, X_0, X_1) = X_0$ and $I(1, X_0, X_1) = X_1$. SI then seek to estimate a time-dependent velocity field $\hat v$ such that the process $(\hat X)_{t\in[0,1]}$ defined as the integration of the ODE \Cref{eq:neuralode} shares the same time-marginal distributions as $(X_t)_{t\in[0,1]}$.
An exact solution $v^\star_t$ is given by the conditional expectation of the path velocities,
\begin{align}\label{eq:vstar}
    v^\star(t, x) = \mathbb{E}\left[\partial_t I(t, X_0, X_1)\mid X_t=x\right]\eqsp.
\end{align}
While this conditional expectation is analytically intractable, $v^\star$ is also a minimizer of a mean-squared regression loss $v^\star \in \argmin_{\hat{v}} \mathcal{L}(\hat{v})$ with
\begin{equation}\label{eq:SIloss}
    \mathcal{L}(\hat{v}) \;=\; \int_0^1 \mathbb{E}\left[\norm{\hat{v}(t, X_t) -\partial_t I(t, X_0, X_1)}^2\right] \rmd t\eqsp.
\end{equation}
This optimization problem allows to build an empirical loss from samples of $\pbase$ and $\pstar$ to learn a velocity field $\hat{v}$ to approximate $v^\star$, completely avoiding the expensive divergence computation during training by maximum likelihood. 
\tcmb{We emphasize that while SI training removes the cost of likelihood evaluation during optimization—thereby enabling the use of highly expressive network architectures—\Cref{eq:ode_likelihood} must still be integrated at sampling time to compute the exact density needed to reweight physical observables.}

Recently, \cite{chen2024flow} extended FM (SI with linear interpolants) to manifold-supported distributions, introducing Riemannian flow matching. This was shortly followed by \cite{wu2025riemannian}, who formulated Riemannian stochastic interpolants (RSI) as the corresponding manifold generalization of SI. In standard Euclidean space, the simplest interpolation is a linear path. The natural generalization to non-Euclidean geometries is to interpolate instead along the shortest path on the manifold, namely a geodesic.

Following \cite{wu2025riemannian}, the geodesic path on the flat torus can be compactly written using the manifold's exponential and logarithmic maps as $I(t, X_0, X_1) = \exp_{X_0}\left(t\log_{X_0}(X_1)\right)$. These geometric operations have direct physical interpretations. The logarithmic map computes the shortest displacement vector between two configurations using the nearest-image convention,
\begin{equation}
    \log_{X_0}(X_1) = \torusmod{X_1 - X_0 + \frac{L}{2}} - \frac{L}{2}\eqsp,
\end{equation}
while the exponential map applies this displacement and wraps the coordinates back into the primary simulation box, 
\begin{equation}
    \exp_{X}(V) = \torusmod{X + V}\,.
\end{equation}
Because the velocity along this geodesic is constant and equal to the initial displacement $\log_{X_0}(X_1)$, the training objective from \Cref{eq:SIloss} simplifies directly to
\begin{align}\label{eq:loss_torus}
    \mathcal{L}(\hat{v}) \;=\; \int_0^1 \mathbb{E}\left[\norm{\hat{v}(t, X_t) - \log_{X_0}\left(X_1\right)}^2\right] \rmd t\eqsp.
\end{align}

\subsection{Equivariant flows}
Learning an accurate approximation of a complex, high-dimensional target distribution with a CNF formulation typically demands highly expressive vector fields, which in turn require vast parameterizations and large amounts of training data. However, the efficiency of this regression task can be drastically improved by guaranteeing by design that the generated distribution $\hat q$ strictly inherits the physical invariances of the target $\pi$.
Equivariant normalizing flows \cite{kohler2020equivariant,satorras2021en_nf} build on this principle. The central result is that coupling a symmetry-invariant base distribution $\pbase$ with an equivariant vector field ensures that the resulting time-dependent density $\hat q_t$ remains strictly invariant under the symmetry group at all times $t \in [0, 1]$.
Formally, a velocity field ${v}$ is equivariant with respect to a symmetry group $G$ if, for any transformation $g \in G$ it satisfies 
\begin{equation}
    v(t,g(X)) = J_g(X)\,v(t,X)\eqsp,
\end{equation}
where $J_g$ denotes the Jacobian of the transformation $g$.

When training via maximum likelihood over the family of equivariant velocity fields with a fixed invariant base distribution, as in \cite{kohler2019equivariantflowssamplingconfigurations,Jung2024normalizingflows}, imposing the equivariance of $\hat{v}$ within the neural network architecture is sufficient: by construction, the generated density inherits the invariance of the base at all times, guaranteeing that the final distribution belongs to the correct invariance class.

In this work, however, we will consider regression-based objectives of the form \Cref{eq:SIloss}, which are computationally cheaper but impose additional structure: the intermediate distributions are no longer free-form, but are instead constrained to follow a prescribed stochastic interpolant path. This raises two questions : (i) are the intermediate densities $(q_t)_{t \in [0,1]}$ of the interpolation process $(X_t)_{t \in [0,1]}$ invariant under the symmetry group? and (ii) if so, is the unique minimizer $v^\star$ of the regression objective equivariant? We address both questions in the following section focusing on the case of particle mixtures.

\section{Equivariant riemannian stochastic interpolants (ERSI)}

\label{sec:ersi}

To correctly model a glass-forming mixture with stochastic interpolants, the generated density must be defined on $\gC = \mathcal{S} \times \torus$ and be invariant under the physical symmetries introduced in \Cref{sec:invariances}. As argued below, this requires not only invariant initial and final distributions, but also that the interpolation path itself be explicitly equivariant. We then show how to adapt an equivariant graph neural network (GNN) architecture \cite{satorras2021en_gnn} to the torus to guarantee that the learned generative process respects these symmetries.

\subsection{Equivariance of the optimal velocity field}

\label{sec:equivariance_optimal_velocity}

Considering the group $\invgroup$ of symmetries of the Boltzmann distribution $\pstar$, we say an interpolation function is $\invgroup$-equivariant if
\begin{equation}
    I(t, g(C_0), g(C_1)) = g(I(t, C_0, C_1))\eqsp, \quad \forall t \in [0,1]\eqsp.
\end{equation}
Physically, this means that applying any symmetry transformation $g \in \invgroup$ to both the initial and final configurations must result in an identical transformation of the entire intermediate trajectory. 

We establish that if both the base distribution $\pbase$ and the target distribution $\pstar$ are $\invgroup$-invariant, employing a $\invgroup$-equivariant interpolant guarantees that the exact target velocity field $v^\star$ defined in \Cref{eq:vstar} is $\invgroup$-equivariant (see \Cref{app:ersi} for the proof).  Moreover, we also show that the time-marginal densities of the interpolation process $q_t$ remain $\invgroup$-invariant at all times. 

These results show that the equivariance of the interpolant is crucial to ensure the consistency of the regression task in \cref{eq:SIloss} where the learned velocity field $\hat{v}$ is also equivariant. To the best of our knowledge, these are original theoretical contributions to the FM and SI literature. Crucially, they show that, unlike in the optimal transport setting of \cite{klein2023equivariant}, aligning the boundary distributions alone is insufficient. 

\subsection{Designing the interpolant}

Based on the requirements derived above, we now specify our choice for the interpolation path. For the full configuration space $\gC$, comprising both the continuous spatial coordinates and the discrete chemical species, the most natural interpolation is the shortest path, or geodesic, on $\gC$ \cite{wu2025riemannian}. Assuming the discrete species are lifted into continuous variables during the flow, the geodesic interpolant decomposes independently across the two spaces (\citep{miller2024flowmm}, see Appendix A.3):
\begin{equation}\label{eq:geodesic_interpolant}
    I(t,(s_0,X_0),(s_1,X_1)) =
    \begin{pmatrix}
        (1-t)s_0 + t s_1 \\
        \exp_{X_0}\big(t\log_{X_0}(X_1))
    \end{pmatrix}\eqsp,
\end{equation}
where the exponential and logarithmic maps are defined in \Cref{sec:si}. The first component represents the geodesic in the Euclidean space of the species variables, which corresponds to a linear interpolation. The second component is the geodesic on the flat torus, which displaces particles along the shortest nearest-image vector between $X_0$ and $X_1$.
We prove in \Cref{prop:geodesic_interpolant_is_equi} that this interpolant is $\invgroup$-equivariant.

In practice, typical models of amorphous materials fix the composition of the mixture. We therefore target the canonical Boltzmann distribution conditioned on a specific particle composition. To enforce this in our framework, we keep the chemical species fixed along the interpolation path, setting $s_0=s_1$.
Consequently, the base distribution $\pbase$ is constructed as the product of a uniform spatial distribution over the torus $\torus^N$ and a uniform distribution over all particle permutations that match the required target composition. During training, a permutation can always be found such that the initial and final species exactly align. As a result, the chemical composition remains constant along the entire interpolation path. 

\subsection{Graph Neural Network architecture}

\label{sec:equi_model}

\begin{figure*}[t!]
    \centering
    \includegraphics[width=1.\linewidth]{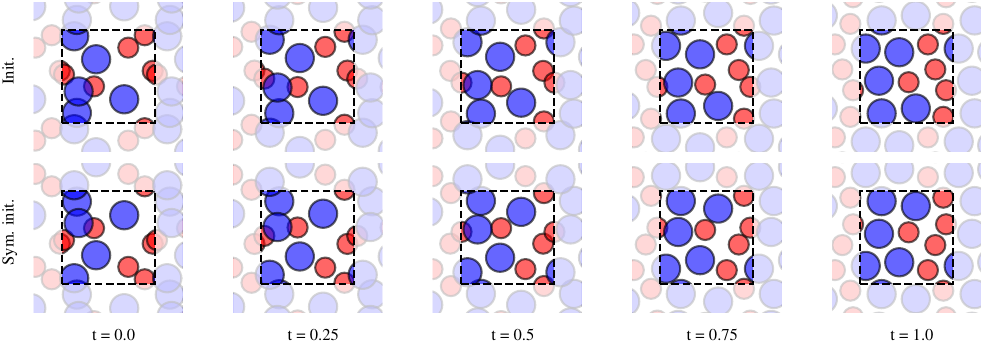}
    \caption{
    Two trajectories of particle configurations generated by integrating \Cref{eq:neuralode} with an equivariant velocity field $\hat v$ started from symmetric initial configurations (first column). The intermediate configurations
    are related by the same initial transformation.}
    \label{fig:two_trajectories}
\end{figure*}

In order to approximate the $\invgroup$-equivariant $v^\star$,  we parameterize the velocity field $\hat{v}$ using a family of GNNs, following recent work on equivariant architectures \citep{kohler2020equivariant, satorras2021en_gnn, jiao2023crystal, miller2024flowmm}.
\tcmb{Compared to prior continuous flow models for amorphous materials relying on Gaussian radial basis functions~\cite{Jung2024normalizingflows}, GNN architectures provides a far more expressive parameterization. 
This increased capacity enables transporting samples from a completely uniform base distribution, whereas simpler architectures only manage to start from a physically structured high-temperature liquid \citep{Jung2024normalizingflows}.}

The construction adapts the architecture in \cite{satorras2021en_gnn} to the torus geometry. Position variables are initialized in the input configuration $\selectparticle{X^{0}}{i} = \selectparticle{X}{i}$ and particle features embed time and particle species $H^0_i = \left(t,s_i\right)$. The GNN architecture then iterates from layer $k$ to layer $k+1$
\begin{align*}
    \quad M^k_{ij} &= \hat{\phi}_e(H_i^k, H_j^k, \torusdist{\selectparticle{X}{i}}{\selectparticle{X}{j}}^2)\eqsp, \\
    P_i^k &= \sum_{i \neq j} \hat{\phi}_m(M_{ij}^k) M_{ij}^k, \quad H^{k+1}_i = \hat\phi_h(H_i^k, P_i^k)\eqsp,\\
     \selectparticle{X^{k+1}}{i} &= \exp_{\selectparticle{X^k}{i}} \left(\sum_{j \neq i} \frac{\log_{\selectparticle{X^k}{j}}\selectparticle{X^k}{i}}{\torusdist{\selectparticle{X}{i}}{\selectparticle{X}{j}} + 1} \hat{\phi}_d(M_{ij}^k)\right)\eqsp,
\end{align*}
where the $\hat \phi_\cdot$ functions are neural networks such that $\hat{\phi}_e$ outputs edges features in $\sR^n$ representing pairwise interactions between particles, $\hat{\phi}_m$ transforms edge messages while preserving dimension before aggregation, $\hat{\phi}_h$ updates the particle features and $\hat{\phi}_d : \mathbb{R}^n \to \mathbb{R}^d$ decodes interaction features into displacements on the torus. For a GNN of depth $K$,
the velocity field is finally obtained as
\begin{align}\label{eq:architecture}
    \hat{v}(t, C) = \begin{pmatrix}
        \mathbf{0}_{N \times d_s} \\
        \log_{\selectparticle{X}{1}}\selectparticle{X^K}{1}, \ldots, \log_{\selectparticle{X}{N}}\selectparticle{X^K}{N}
    \end{pmatrix}\eqsp,
\end{align}
where the null component corresponds to the preserved species-composition along the interpolation.

Unlike the linear $E(n)$ symmetries considered by the original architecture of \cite{satorras2021en_gnn}, the symmetry group $\invgroup$ is nonlinear. Extending equivariance from linear Euclidean spaces to $\invgroup$ requires some care.
We establish this rigorously in \Cref{app:ersi}, proving that this modified architecture defines a Lipschitz-bounded, $\invgroup$-equivariant velocity field on $\gC$.
To illustrate this, \Cref{fig:two_trajectories} compares two ODE trajectories whose initial configurations are related by a symmetry transformation. Because the learned velocity field is equivariant, this relationship is preserved throughout the entire integration. At any intermediate time $t$, the configurations in the two paths remain related by the exact same initial symmetry operation. 

Finally, the combination of a symmetry-preserving interpolation path with a geometry-aware equivariant architecture defines the generative modeling framework used in this work, which we refer to as the Equivariant Riemannian Stochastic Interpolant (ERSI). 

\section{Models and methods}

\label{sec:methods}

In this section, we describe the physical models and additional implementation choices made in our evaluation of the ERSI used in \Cref{sec:numerical-results}.

\subsection{Two model systems}

\label{sec:systems}

We consider two-dimensional systems composed of $N=10$ and $N=44$ particles in a square box with periodic boundary conditions at constant temperature $T$ in the canonical ensemble.

The first system \citep{Bernu1987soft, Perera1999Stability}, is a binary (50:50) mixture interacting via the inverse power law interaction potential
$$
    \mathrm{W}_{\mathrm{IPL}}(s_1, s_2, r) =
    \begin{cases}
    \epsilon \left(\dfrac{\sigma_{s_1 s_2}}{r}\right)^{12} + W_0, & r < r^{\mathrm{cut}}_{s_1s_2}, \\[4pt]
    0, & \text{otherwise},
    \end{cases}\eqsp,
$$
with $\sigma=\begin{pmatrix} 1.0 & 1.2 \\ 1.2 & 1.4 \end{pmatrix}$ and $\epsilon = 1$. The value $W_0$ is chosen to shift the potential continuously to zero at $r^{\mathrm{cut}}_{s_1s_2} = 2.5 \sigma_{s_1s_2}$.
The number density is fixed to $N/L^2 = 0.5$ and the temperature to $T=0.1$, corresponding to a dense fluid state.

The second system \cite{jung2023predicting, jung2025numerical} is a  ternary $\left(5/11,\,3/11,\,3/11\right)$ variation of the Kob-Andersen mixture Lennard-Jones \cite{Kob1995testing} introduced for efficient simulations. The interaction is defined as
\begin{widetext}
\begin{equation*}
    \mathrm{W}_{\mathrm{KA}}(s_1, s_2, r) =
    \begin{cases}
        4\epsilon_{s_1 s_2} \mathrm{W}_{\mathrm{LJ}}(s_1, s_2, r) 
        + \mathrm{W}_0 
        + \mathrm{W}_2\left(\frac{r}{\sigma_{s_1 s_2}}\right)^2 
        + \mathrm{W}_4\left(\frac{r}{\sigma_{s_1 s_2}}\right)^4\eqsp,
        & r < r^{\mathrm{cut}}_{s_1s_2}, \\[6pt]
        0, & \text{otherwise}\eqsp,
    \end{cases}
\end{equation*}
\end{widetext}
with 
\begin{gather*}
     \mathrm{W}_{\mathrm{LJ}}(s_1, s_2, r) = \left[\left(\frac{\sigma_{s_1 s_2}}{r}\right)^{12} - \left(\frac{\sigma_{s_1 s_2}}{r}\right)^{6}\right], \\
     \epsilon = \begin{pmatrix}
        1.0 & 1.5 & 0.75 \\
        1.5 & 0.5 & 1.5 \\
        0.75 & 1.5 & 0.75
    \end{pmatrix}, ~~
    \sigma = \begin{pmatrix}
        1.0 & 0.8 & 0.9 \\
        0.8 & 0.88 & 0.8 \\
        0.9 & 0.8 & 0.94
    \end{pmatrix},
\end{gather*}
$r^{\mathrm{cut}}_{s_1s_2} = 2.5 \sigma_{s_1s_2}$, and correction terms $(\mathrm{W}_0, \mathrm{W}_2, \mathrm{W}_4)$ chosen as in \cite{jung2023predicting}.
The number density is fixed at $N/L^2=1.192075$, and we consider two temperatures, $T = 1.0$, which marks the onset of glassy dynamics and $T = 0.32$, representing a deeply supercooled state.

\subsection{Training dataset}

\label{sec:dataset}

We assume access not only to the target energy function $U$, but also to a dataset of configurations distributed according to the equilibrium measure $\pstar$. This auxiliary data set enables an initial learning phase in which a generative model is trained to define a proposal distribution that can later be corrected via importance sampling or even improved in an adaptive loop \cite{gabrieAdaptiveMonteCarlo2022}.

We generated the training datasets with the Metropolis–Hastings Monte Carlo algorithm~\citep{metropolis1953}. Starting from particles uniformly distributed in the box, the update kernel consists in selecting one particle randomly with equal probability and attempting a displacement drawn from a Gaussian distribution centered at the current position with standard deviation $0.065$~\citep{dfrenkel96:mc}. The move is then accepted or rejected according to the standard Metropolis criterion. We define one unit of time as $N$ (number of particles) attempted moves. 
For the binary system at temperature $T=0.1$ and for the ternary system at temperature $T=1.0$, we run 100 independent chains, each initialized from a different random configuration, for $10^4$ time units to reach equilibrium. In this time scale, the potential energy rapidly relaxes to a steady value, and the self-intermediate scattering function~\citep{berthier2011theoretical} (a standard time-correlation function used to quantify structural relaxation in liquids) decays to zero within $10^4$ time units, confirming that the system is fully equilibrated. From these equilibrated configurations, each chain was then propagated for an additional $10^7$ time units, storing one configuration every $10^4$ steps.

For the ternary model at temperature $T=0.32$, we augmented the displacement moves with particle swap moves~\citep{ninarello2017models}, applied with probability $p_{\text{swap}}$. In a swap proposal, two particles of different species are randomly selected and the two species are exchanged. The proposal is then accepted or rejected according to the usual Metropolis criterion. Swap moves are known to dramatically accelerate equilibration in this specific system~\citep{jung2023predicting}. In this setting, we run 100 independent chains for $5\times 10^4$ time units to reach equilibrium, followed by $5\times 10^7$ time units for data collection, storing one configuration every $5\times 10^4$ steps.

Both procedures yield a total of $10^5$ uncorrelated equilibrium configurations, which we use as training data. The datasets are available on Zenodo \cite{grenioux_2025_17966995}.

\subsection{Equivariant optimal transport}

To improve training efficiency and shorten the transport trajectories, recent works replace the independent endpoint sampling $(X_0, X_1) \sim \pbase \otimes \pstar$ with a coupling closer to the optimal transport (OT) coupling $(X_0, X_1) \sim \Pi(\pbase, \pstar)$ \citep{tong2024improving, Albergo2024ot}. To ease computations, the OT problem is solved between mini-batches of samples from $\pbase$ and $\pstar$  used when computing the training objective (\ref{eq:loss_torus}), as described in \citep{Fatras2021Unbalanced}, which is an idea widely adopted for particle systems \citep{klein2023equivariant, Song2023equivariant, irwin2024efficient}. Given two configurations $C_0=(\tilde  s , X_0)$ and $C_1=(\tilde  s, X_1)$ with identical composition, each is partitioned by species, and within each group the Hungarian algorithm is applied using $\mathrm{d}_{\torus}$ (see \Cref{eq:torus:dist}) as cost. This OT-based matching, consistent with the particle-permutation invariance, shortens the interpolation paths significantly. While previous work \cite{klein2023equivariant} has accounted for additional invariances in the OT coupling, we found that further applying discrete box rotations to optimally align the spatial orientations of $C_0$ and $C_1$ before interpolation substantially increased the computational overhead without producing measurable improvements in the thermodynamic observables. Therefore, in the results presented below, the OT coupling is restricted to particle permutations.

\subsection{Baseline methods}

\label{sec:benchmark_methods}

To isolate the respective roles of spatial geometry and coordinate invariances, we evaluate the performance of ERSI against two reference baselines designed to separately relax these two constraints.

The first reference model is a geometry-aware but symmetry-agnostic baseline, which we refer to as Riemannian Stochastic Interpolants (RSI). This model accounts for periodic boundary conditions by employing the geodesic interpolation from \Cref{eq:geodesic_interpolant} directly on the flat torus $\torus^N$ as ERSI, but it does not structurally enforce the invariances of the group $\invgroup$. The velocity field is parameterized by a standard multilayer perceptron (MLP) that maps raw coordinate arrays without permutation or translation constraints, allowing one to test whether a symmetry-agnostic network can implicitly infer these physical conservation laws from the training configurations.


\tcmb{The second reference model is a geometry-agnostic Equivariant Flow Matching (EFM) baseline, following the general formulation of \cite{klein2023equivariant}. This method uses the equivariant graph neural network introduced in \cite{satorras2021en_gnn} to enforce the full $E(d)$ symmetry group of permutations, translations, and rotations within flat Euclidean space $\sR^{Nd}$. Crucially, this baseline implicitly targets a different Boltzmann distribution than ERSI: its interaction potential is evaluated using the Euclidean distance $\|\selectparticle{X}{i}- \selectparticle{X}{j}\|$ rather than the torus distance $\torusdist{\selectparticle{X}{i}}{\selectparticle{X}{j}}$ (Eq.~(1)), and it is trained with a linear interpolant that minimizes this Euclidean distance between $X_0$ and $X_1$, without constraining the interpolation trajectory to remain on the manifold $\gM$. By neglecting the toroidal topology of the periodic box in both its target distribution and its interpolant, this baseline allows us to directly monitor the physical consequences of ignoring periodic boundary conditions during the generative process.}

For a fair comparison, the total number of trainable parameters across the three generative models is matched, and all models are trained for the same number of epochs (see \Cref{app:impl_details} for details). 

\subsection{Relevant physical observables}

\label{sec:observables}

To assess the physical validity of the generated configurations, we monitor three standard structural and thermodynamic observables. 
First, we compute the average potential energy 
\begin{equation}
\langle U\rangle =\mathbb{E}_{\pstar}[\mathrm{U}_\star]
\end{equation}
of the generated configurations. Second, we evaluate the specific heat capacity at constant volume,
\begin{equation}
c_V = \frac{\mathbb{E}_{\pstar}[\mathrm{U}_\star^2] - \mathbb{E}_{\pstar}[\mathrm{U}_\star]^2}{N k_{\rm B} T^2} \eqsp,
\end{equation}
which probes energy fluctuations and is notoriously sensitive to sampling inaccuracies \citep{flenner2006hybrid,Jung2024normalizingflows}. 
Third, we determine the liquid structure via the radial distribution function \citep{barrat2003basic},
\begin{equation}
g(r) = \mathbb{E}_{\pstar}\left[ \frac{L^2}{N^2} \sum_{i=1}^N \sum_{j \neq i}^N \delta\left(r - \torusdist{\selectparticle{X}{i}}{\selectparticle{X}{j}}\right) \right] \eqsp,
\end{equation}
which measures the density profile around a reference particle.

The statistical efficiency of the self-normalized reweighting protocol is determined by the configuration-space overlap between the proposal density $\hat q$ and the target measure $\pstar$. If the proposal distribution deviates significantly from the target, the importance weights become highly unevenly distributed, causing the variance of the estimator to diverge. We track this sampling quality by measuring the normalized effective sample size (ESS), defined via the inverse participation ratio of the normalized weights $\overline w(C_i) = e^{-\mathrm{U}_\star(C_i)/k_{\rm B} T}/\hat{q} (C_i)$:
\begin{equation}
\mathrm {ESS}(C_{1:R})=\frac{\left(\sum_{i=1}^R\overline w(C_i)\right)^2}{R \sum_{i=1}^R\overline w_i(C_i)^2}\eqsp, \text{ for i.i.d. } C_i \sim \hat{q} \, .
\end{equation}
The value of the ESS is bounded between $1/R$ and 1. 
This metric is inversely proportional to the relative variance of the importance weights, directly reflecting the asymptotic variance of the self-normalised estimator in \Cref{eq:snis}~\cite{Agapiou2017Importance}. A low ESS indicates an imbalanced weight distribution where the physical averages are dominated by a very small fraction of rare configurations, signaling that the proposal distribution does not generate a sufficient number of relevant states of the target system.

\section{Numerical results}

\label{sec:numerical-results}

\subsection{Role of geometric and symmetry constraints in the binary mixture}

\label{sec:ablation}

\begin{figure*}[t!]
    \begin{center}
        \includegraphics[width=1.\linewidth]{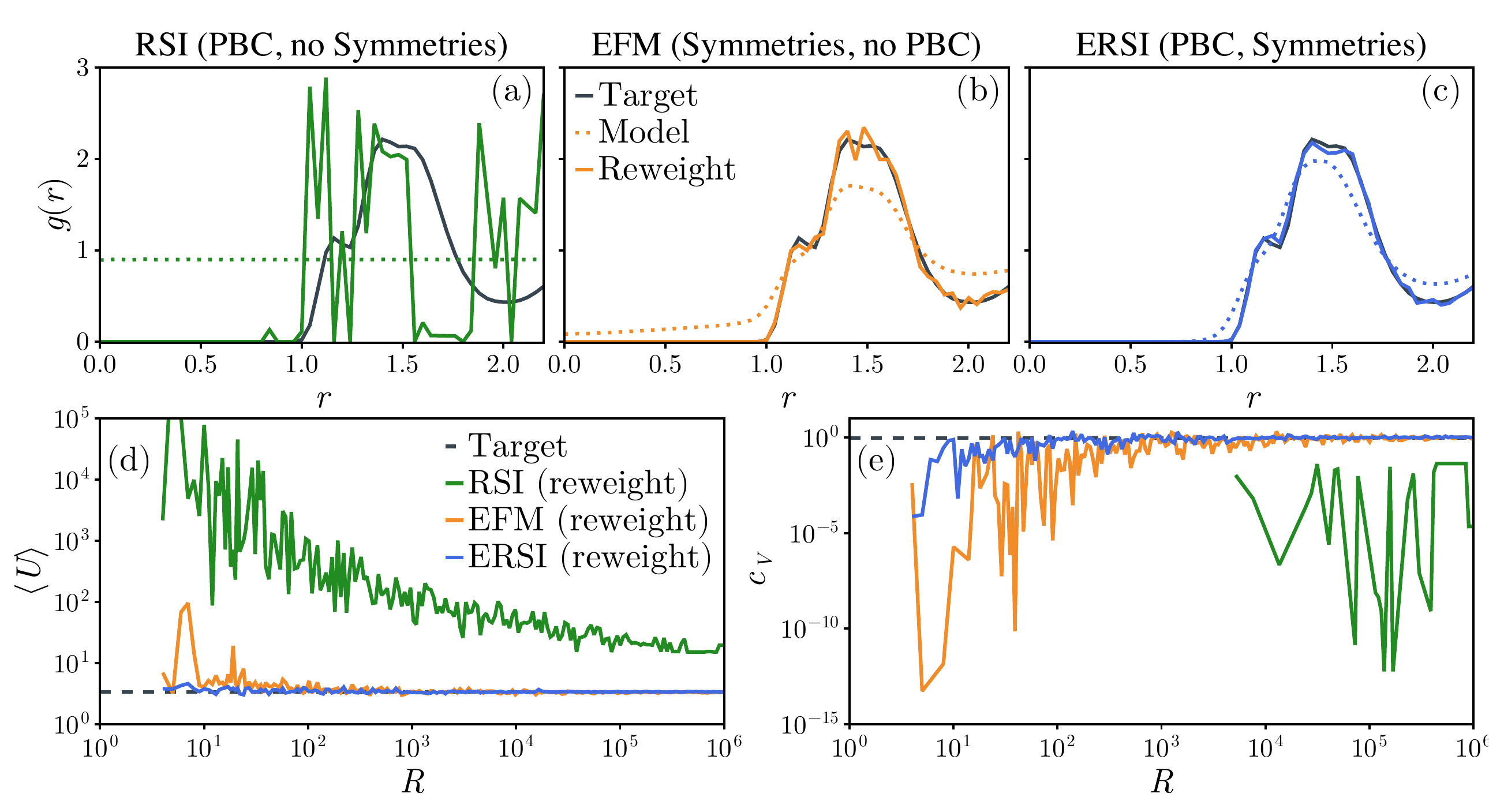} 
    \end{center}
    \caption{Physical observables for the $N=10$ binary mixture. Radial distribution function $g(r)$ computed for (a) RSI, (b) EFM, and (c) ERSI. Dotted, solid colored and solid black lines represent uncorrected model proposals, reweighted estimates, and reference target values, respectively. Convergence of (d) the mean potential energy $\langle U\rangle$ and (e) the specific heat $c_V$ as a function of the number of generated configurations $R$.}   
    \label{fig:N10} 
\end{figure*}

\begin{figure}[t!]
    \centering
    \includegraphics[width=1.\linewidth]{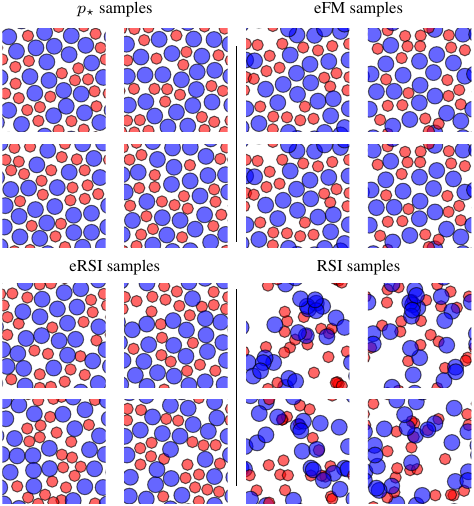}
    \caption{Typical particle configurations for the $N=44$ binary mixture. Snapshots obtained from the target distribution ($\pstar$) are compared with proposals generated by the EFM, RSI, and ERSI models. The symmetry-agnostic RSI baseline fails to move particles resulting in random configurations, while the geometry-agnostic EFM model generates realistic local packing but exhibits particle overlaps across the periodic boundaries. The ERSI framework simultaneously preserves both the local liquid structure and the boundary constraints.}
    \label{fig:samples_N44}
\end{figure}

\begin{figure*}[t!]
    \centering
    \includegraphics[width=1.\linewidth]{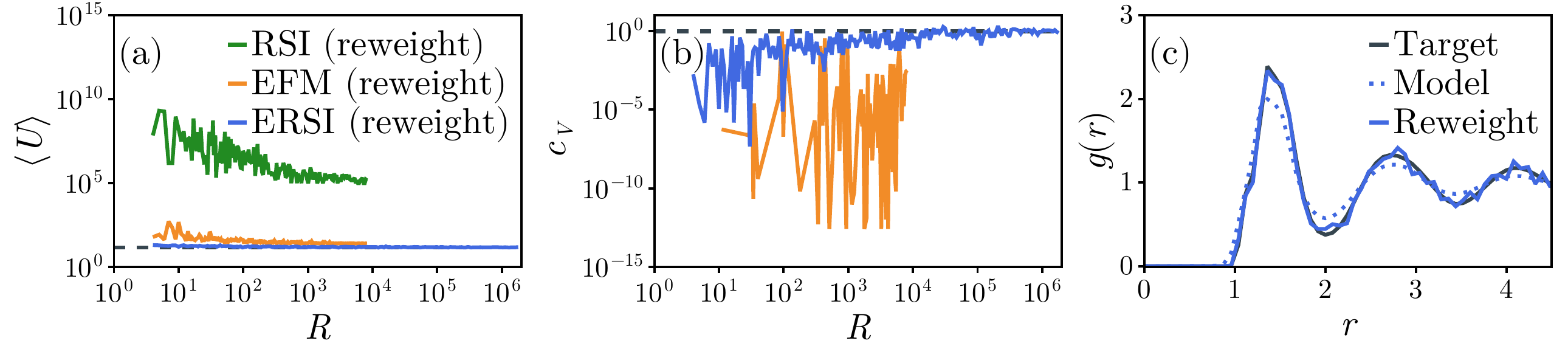} 
    \caption{Physical observables for the $N=44$ binary mixture. Convergence of (a) the mean potential energy $\langle U\rangle $ and (b) the specific heat $c_V$ as a function of the number of generated configurations $R$. (c) Radial distribution function $g(r)$ computed from $1.8\times10^6$ configurations generated by the ERSI framework. Reference baselines are truncated or omitted where unphysical boundary configurations or numerical instabilities prevent reliable estimate.}
    \label{fig:N44} 
\end{figure*}

We evaluate the performance of the Equivariant Riemannian Stochastic Interpolant (ERSI) by comparing it against the two reference baselines introduced in \Cref{sec:benchmark_methods}: the geometry-agnostic but symmetry-aware EFM model, and the geometry-aware but symmetry-agnostic RSI model. The efficiency of each architecture is verified by tracking the convergence of the quantities introduced in \Cref{sec:observables} as a function of the number of generated samples $R$. For all models, expectations are computed both directly from generated configurations and by reweighting the generated proposal samples through the importance sampling estimator in \Cref{eq:snis}. Reference equilibrium values, denoted as the target, are obtained directly from the training set. 

We first examine a binary mixture containing a small number, $N=10$, particles. As shown in \Cref{fig:N10}, the symmetry-agnostic RSI baseline fails to capture meaningful structure, yielding a completely flat radial distribution function $g(r)$. This indicates that a simple architecture cannot implicitly infer the particle permutation and translation invariances from the training data alone. As a consequence, reweighted estimates of physical quantities are incorrect even for large number of samples. In contrast, the geometry-agnostic EFM performs better: the average potential energy $\langle U\rangle $ and the specific heat $c_V$ can be recovered through reweighting, along with a noisy estimate of $g(r)$. The proposed ERSI framework displays the highest sample efficiency, converging rapidly to the exact target values for all quantities with significantly fewer samples than the EFM model.

To investigate how these architectural choices scale with the system size, we increase the number of particles to $N=44$. Moving to a larger system increases the dimensionality and complexity of the configuration space, exposing the scaling limitations of the baselines. Typical snapshots of the generated configurations are presented in \Cref{fig:samples_N44}, alongside typical training configurations from $\pstar$. RSI fails to move particles significantly, producing configurations that remain close to the uniform base distribution, hence its poor performance. Augmenting the training data with random actions from $\invgroup$ to help the RSI model learn the invariances was attempted, but did not improve performance. EFM generates more realistic configurations, but many particle overlaps occur near the box edges because periodic boundary conditions are ignored. These overlaps cause three issues: (i) averages obtained without reweighting display unphysical characteristics, such as a much higher value of $g(r<1)$, (ii) many samples are discarded during the correction, as configurations with overlapping particle pairs receive zero weight in \Cref{eq:snis}, and (iii) the likelihood of overlaps at the boundaries increases with system size, limiting the scalability of the EFM model. In contrast, configurations produced by ERSI correctly incorporate periodic boundary conditions and do not suffer from this problem, resulting in realistic configurations. 
This is quantitatively confirmed in \Cref{fig:N44}, where ERSI is the only method that allows to recover correct estimates of physical observables. 

\begin{figure*}[t!]
    \centering
    \includegraphics[width=1.\linewidth]{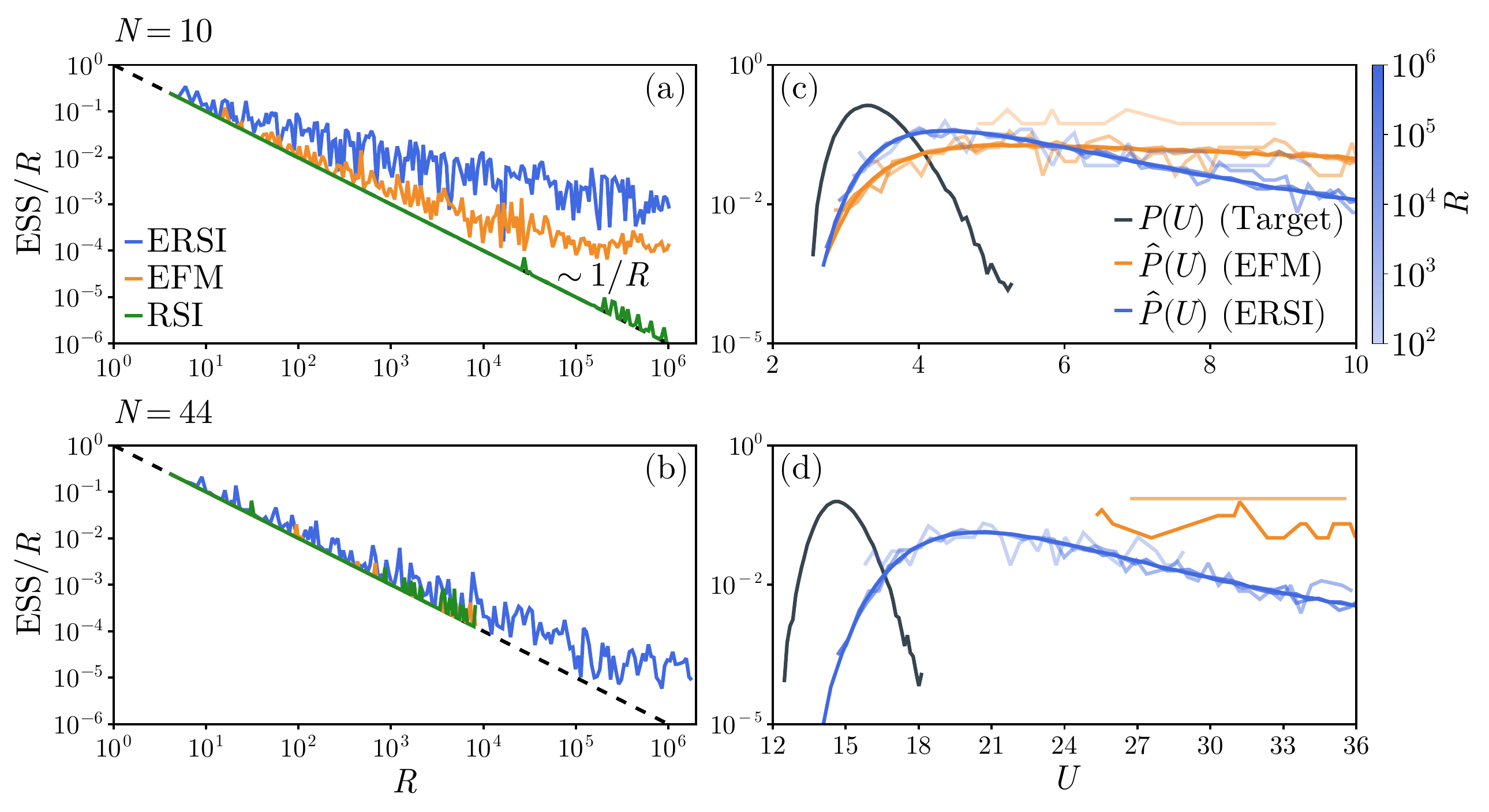} 
    \caption{Sampling efficiency and energy distribution overlap.
    (a–b) Normalized effective sample size $\mathrm{ESS}/R$ as a function of the number of generated configurations $R$ for (a) $N=10$ and (b) $N=44$ particles. Results are shown for the symmetry-agnostic RSI baseline (green), the geometry-agnostic EFM baseline (orange), and the proposed ERSI framework (blue). Dashed black lines indicate the asymptotic $1/R$ scaling regime where a single configuration dominates the reweighted estimator. 
    (c–d) Corresponding potential energy histograms $P(U)$ showing the overlap with the equilibrium target reference (dark grey solid curve) for different sample volumes $R$ (indicated by the color bar). For clarity, unphysical high-energy configurations exceeding twice the maximum target energy are omitted from the histograms; this truncation discards $100\%$, $\approx 84\%$, and $\approx 3\%$ of the proposals generated at $N=10$ by RSI, EFM, and ERSI, respectively. For $N=44$, the maximum sample size for RSI and EFM is bounded at $R=8\times10^3$ as generating more samples for these poor baselines does not provide further statistical benefit.
    } 
    \label{fig:iplESS} 
\end{figure*}

The difference in reweighting efficiency can be quantified by monitoring the normalised effective sample size ($\mathrm{ESS}/R$) as a function of the number of generated proposals $R$. In our simulations, we observe that the normalised ESS initially decays as $1/R$ over a wide range of sample sizes before eventual stabilisation on a finite plateau at large $R$. This transient regime implies that for small $R$, the importance sampling estimator is structurally dominated by a singular configuration carrying nearly all the statistical weight.
For the small system ($N=10$), Fig.~\ref{fig:iplESS}(a) shows that the crossover scale $\overline{R}$ at which the estimator escapes this $1/R$ scaling is strongly model-dependent. For the symmetry-agnostic RSI baseline, no crossover is observed, and the estimator remains trapped in the $1/R$ regime over the entire sampling range. For the geometry-agnostic EFM model, the plateau is reached at $\overline{R} \approx 10^4$, whereas the proposed ERSI framework achieves stabilization much earlier, at $\overline{R} \approx 10^3$, demonstrating a significantly higher sample efficiency. 

To rationalize this observation, we compare the energy distributions of samples from the three models with that of the target, shown in \Cref{fig:iplESS}(c). For clarity, we discard configurations with energies larger than twice the maximum observed in the target dataset. The discarded fraction is 100\% for RSI (hence no data are shown for this model), about 84\% for EFM, and about 3\% for ERSI. For small $R$, the two distributions $P(U)$ (target) and $\hat P(U)$ (model) overlap poorly: $\hat P(U)$ is concentrated at higher energies and only begins to penetrate the region of significant $P(U)$ weight from the right tail. As $R$ increases, rare samples from this region appear with very small $\hat P(U)$ but relatively large $P(U)$, producing very large reweighting factors. A few such configurations dominate the averages, driving the $1/R$ behavior of the ESS. Once $R$ exceeds $\overline{R}$, $\hat P(U)$ has infiltrated sufficiently deep into the bulk of $P(U)$ so that ratios $P(U)/\hat P(U)$ are less extreme, and the ESS stabilizes to a plateau. 
We repeat the analysis for $N=44$ in \Cref{fig:iplESS}, using up to $R=8{\times}10^3$ samples for RSI and EFM and up to $R=10^6$ samples for ERSI.
This choice is practical: since the baselines produce heavily overlapping configurations that carry vanishing statistical weight, extending their sample sizes provides no additional statistical benefit.
The results are qualitatively the same, although EFM samples are significantly worse compared to $N=10$ due to overlapping particles at the boundaries of the physical domain (see Fig.~\ref{fig:samples_N44}).

\subsection{Equilibrium sampling of the ternary mixture}

Having demonstrated that incorporating explicit geometric and symmetry constraints into the generative model significantly increases sampling efficiency for a binary mixture, we extend our analysis to a more complex glass-forming model: the ternary mixture described in \Cref{sec:systems}. This system is advantageous because equilibrium reference configurations can be produced at temperatures deep into the supercooled regime using swap Monte Carlo moves (see \Cref{sec:dataset}). This enables the generation of large training datasets at low temperatures where conventional molecular dynamics struggles to equilibrate, which allows for assessing the model's performance in a regime where relaxation dynamics is very slow. 

\begin{figure*}[t]
    \centering
    \includegraphics[width=1.\linewidth]{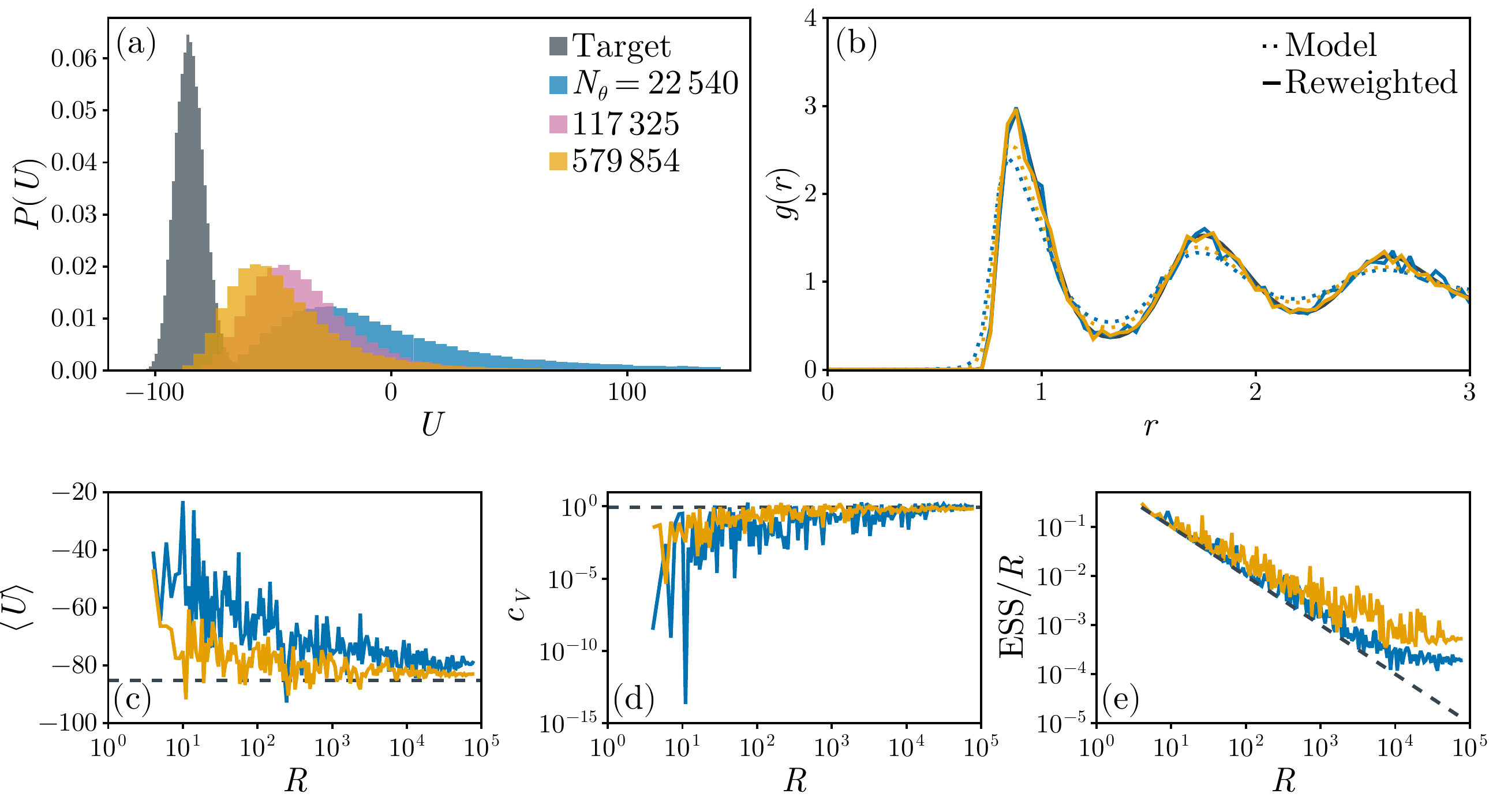} 
    \caption{Comparison of model sizes for the ternary mixture at $T=1.0$. 
    (a) Potential energy distribution $P(U)$ comparing the target reference (gray) with raw model proposals generated by architectures of varying number of trainable parameters $N_\theta$.
    (b) Radial distribution function $g(r)$ comparing the target reference (gray), the unweighted model proposals (dotted lines), and the reweighted estimates (full lines) across different model sizes.
    (c-d) Convergence of the reweighted mean potential energy $\langle U\rangle$ (c) and specific heat $c_V$ (d) as a function of the number of configurations $R$. 
    (e) Evolution of the normalized effective sample size $\mathrm{ESS}/R$ with $R$.
    }   
    \label{fig:model_size} 
\end{figure*}
We first focus on a state point at temperature $T = 1.0$, which corresponds approximately to the onset temperature of slow dynamics for this mixture. In this state point, we systematically investigate the effect of network capacity on generative performance, evaluating whether scaling up the number of model parameters improves the convergence of the reweighted estimators.
In Fig.~\ref{fig:model_size}(a), we plot the potential energy distribution $P(U)$ of model proposals generated by networks of varying sizes against the target reference ensemble. 
Figure~\ref{fig:model_size}(b) shows the corresponding radial distribution function $g(r)$, comparing the unweighted proposals (dotted lines) and the reweighted estimates (solid lines) against the target reference (gray). These data confirm that increasing the network capacity (ranging from approximately $22\times10^3$ parameters for the smallest model to $58\times10^4$ parameters for the largest) systematically leads to a closer agreement with the target structure. 
\tcmb{However, this performance is saturating: while increasing the parameter count from $22\times10^3$ to $117\times 10^3$ yields a substantial improvement in the unweighted mean potential energy, a further five-fold increase yields a much smaller gain. A rough extrapolation of this trend suggests that bridging the remaining gap to the average target energy purely by scaling up the network capacity would require an architecture with more than $10^{7}$ parameters, making brute-force scaling practically unfeasible. Nevertheless,} the improved quality of these proposal samples directly accelerates the convergence of the reweighted quantities. As displayed in Figs.~\ref{fig:model_size}(c) and (d), the estimators for both the mean potential energy $\langle U\rangle$ and the specific heat $c_V$ converge towards the target values at lower values of $R$ with larger architectures. This improved efficiency is also confirmed by the evolution of the normalized effective sample size in Fig.~\ref{fig:model_size}(e), where the crossover scale $\overline R$ at which the ESS departs from the $1/R$ regime and stabilizes on a finite plateau is lowered by a factor of 10 for the largest network compared to the smallest network. This improved performance, however, comes at the cost of approximately $1.5$ times slower training and approximately $6$ times slower sampling on an NVIDIA A100 GPU.

\begin{figure*}[t!]
    \centering
    \includegraphics[width=1.\linewidth]{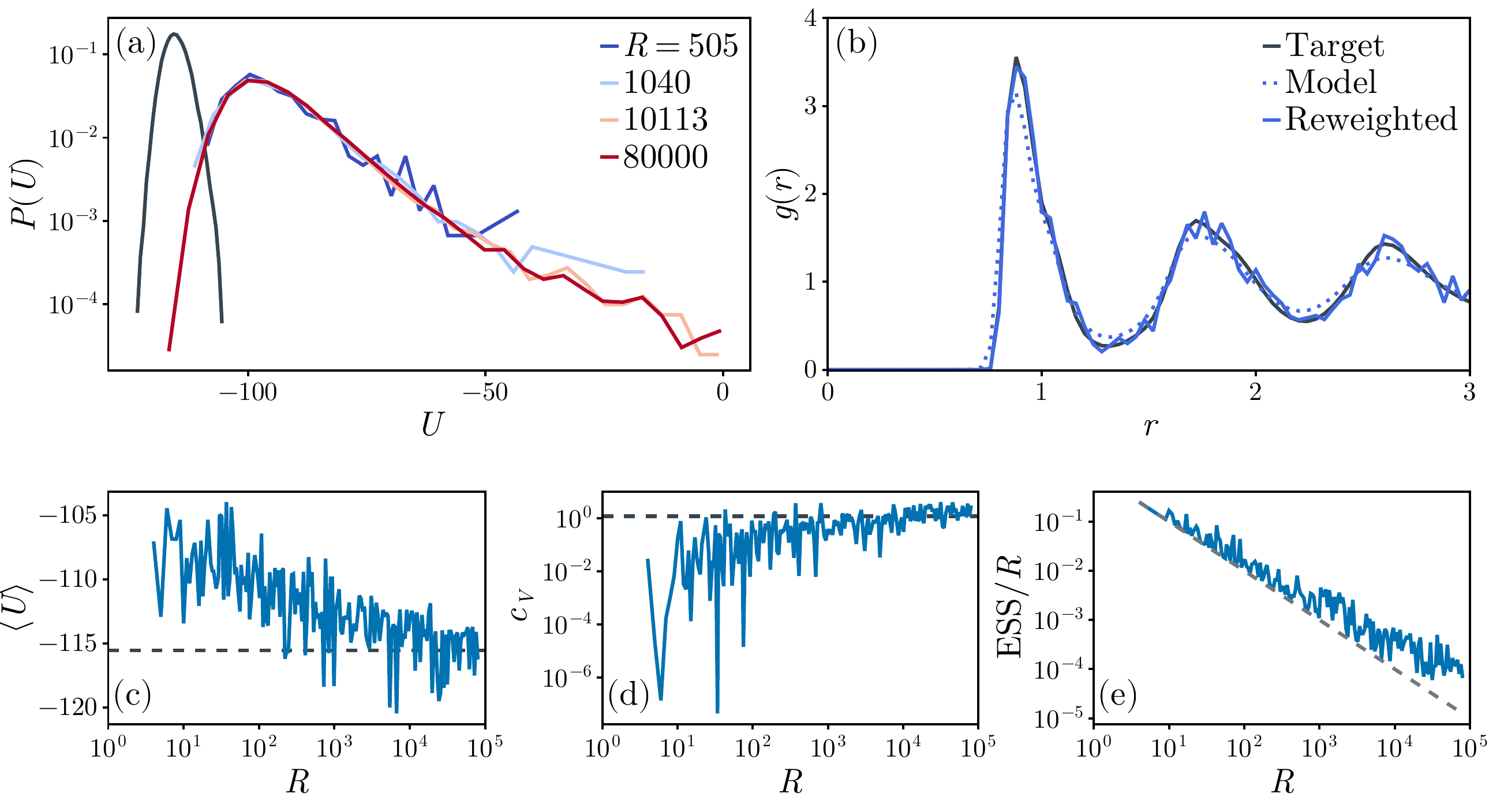} 
    \caption{Sampling efficiency of the ternary mixture in the supercooled regime at $T=0.32$. 
    (a) Potential energy distribution $P(U)$, comparing the target reference (gray) with the raw model proposals for different sample sizes $R$.
    (b) Radial distribution function $g(r)$ comparing the equilibrium reference (gray), the unweighted model proposals (dotted blue), and the reweighted estimates (solid blue).
    (c–d) Convergence of the reweighted mean potential energy $\langle U\rangle$ (c) and specific heat $c_V$ (d) as a function of the number of generated configurations $R$. 
    (e) Evolution of the normalized effective sample size with $R$ settling at approximately $10^{-4}$.
    }  
    \label{fig:jbb32} 
\end{figure*}

We now turn to a significantly more challenging regime by lowering the temperature to $T = 0.32$, which corresponds to a supercooled state point where the structural relaxation time of the system increases by a factor of approximately 300 compared to the onset temperature at $T = 1.0$. Numerical tests indicated that the network architecture with $58\times10^4$ parameters fails to achieve sufficient accuracy to allow for reweighting. Consequently, we scale up the network capacity to an architecture with approximately $10^6$ parameters.

The results are summarized in Fig.~\ref{fig:jbb32}. As shown in the potential energy distribution $P(U)$ in Fig.\ref{fig:jbb32}(a), the overlap between the model proposals and the equilibrium target is not large. The two distributions intersect only through their tails, where the low-energy proposals match the high-energy target samples.
This poor overlap leads to low statistical efficiency when reweighting physical quantities. The radial distribution function $g(r)$ in Fig.~\ref{fig:jbb32}(b) shows that the reweighted estimator successfully corrects the unweighted model proposals to match the target curve, though the resulting estimate remains somewhat noisy. Figures~\ref{fig:jbb32}(c) and (d) show the convergence of the reweighted estimators for the mean potential energy $\langle U\rangle$ and the specific heat $c_V$ as a function of $R$. Both quantities exhibit relatively large fluctuations up to large sample sizes $R \gtrsim 10^4$. 
This slow convergence across both structural and thermodynamic quantities is reflected in the normalized effective sample size $\mathrm{ESS}/R$ shown in Fig~\ref{fig:jbb32}(e), which follows the $1/R$ decay over most of the sample sizes, before plateauing to approximately $10^{-4}$ above $\overline{R} \approx 10^4$. This trend indicates that the reweighted averages are dominated by a small fraction of rare, low-energy configurations, and this explains the statistical noise in the pair correlation function. 

\subsection{Uncovering a fundamental limitation of continuous flows}

\subsubsection{Pre-filtering configurations before evaluating the density} 

Boltzmann Generators based on continuous normalizing flows offer an elegant route to exact reweighting, but their practical application to challenging scenarios reveals a severe computational bottleneck. Generating configuration by integrating \Cref{eq:neuralode} is relatively cheap, but evaluating the exact generated density $\hat{q}(C)$ requires solving a coupled ODE system with \Cref{eq:ode_likelihood}, which requires computing the divergence of the parameterized velocity field at each integration step. This divergence computation scales quadratically with the number of particles, and using the approximation trick of Hutchinson was shown to be too crude to perform statistical reweighting \cite{hoffmann2026boltzmann}. As a result, in challenging settings, where many samples are required to populate the rare low-energy states, the integration of the coupled ODE giving the likelihood becomes computationally prohibitive.

To circumvent this limitation, we can leverage the fact that, in our application, the majority of generated proposals have much higher energies compared to target samples and thus carry vanishingly small importance weights that do not contribute to equilibrium averages. We can thus implement a simple filtering strategy, where we could generate a large number of configurations using \Cref{eq:neuralode} without computing their densities. We then evaluate their potential energies $U$ (this is numerically cheap) in order to restrict the expensive density evaluation to a small fraction $\alpha$ corresponding to the lowest-energy samples, i.e. to samples that have a chance to contribute to the equilibrium average. The density $\hat q$ for these selected configurations is obtained by integrating \Cref{eq:ode_likelihood} backward from $t=1$ to $t=0$. A similar approach was formally introduced in \cite{tan2024scalable}, where analytical bounds for the resulting truncation error were derived. Here, we evaluate the robustness of this protocol empirically using our results for the ternary mixture at $T=0.32$. Specifically, we fix a target threshold for the retained partition function mass, $\mathcal{Z}_\alpha/\mathcal{Z} \approx \sum_{i=1}^{R_\alpha}\overline{w}(C_i)/\sum_{i=1}^R\overline{w}(C_i)$, and track the minimum fraction of configurations $\alpha$ required to satisfy this threshold across different total sample sizes $R$.
As shown in Fig.~\ref{fig:filtering}(a), retaining only the lowest-energy $10\%$ of the generated configurations consistently preserves more than $99.9\%$ of the total partition function across sample sizes $R$. The corresponding truncated potential energy histogram and the log-weight distributions before and after filtering are displayed in Figs.~\ref{fig:filtering}(b) and (c) for a total of $8\times10^4$ samples. These distributions confirm that the filtering process safely removes configurations with negligible statistical weight without altering the largest weights, which would otherwise introduce an uncontrolled bias.

\begin{figure*}[t!]
\centering
\includegraphics[width=1.\linewidth]{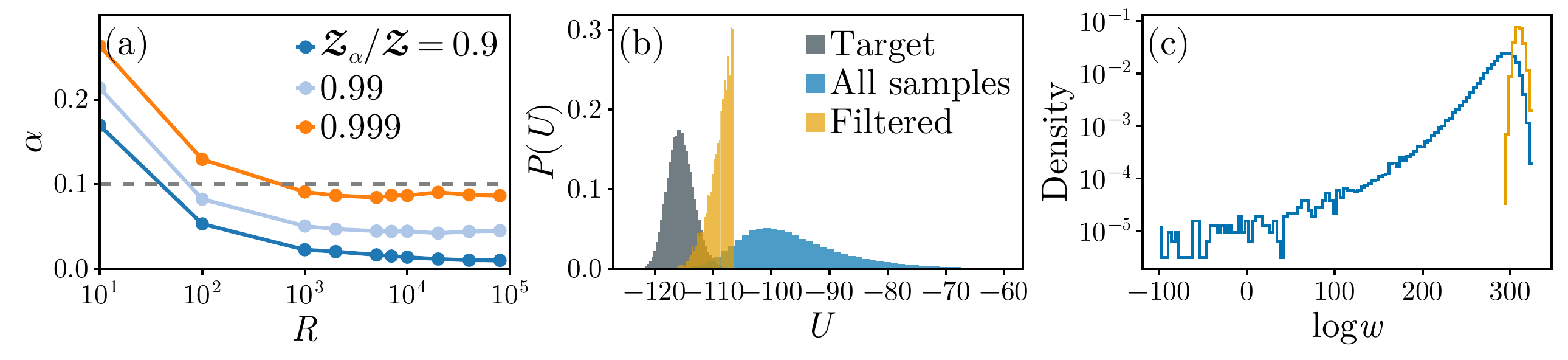}
\caption{Energy-filtering protocol.
(a) Fraction $\alpha$ of lowest-energy samples needed to preserve a given fraction $\mathcal Z_\alpha/\mathcal Z$ of the estimated target partition function as a function of the total number of samples $R$.
(b) Truncated potential energy histogram showing the discarded high-energy configurations.
(c) Distribution of the unnormalized importance weights before and after applying the $10\%$ filtering threshold.}
\label{fig:filtering}
\end{figure*}

\begin{figure*}[t!]
\centering
\includegraphics[width=1.\linewidth]{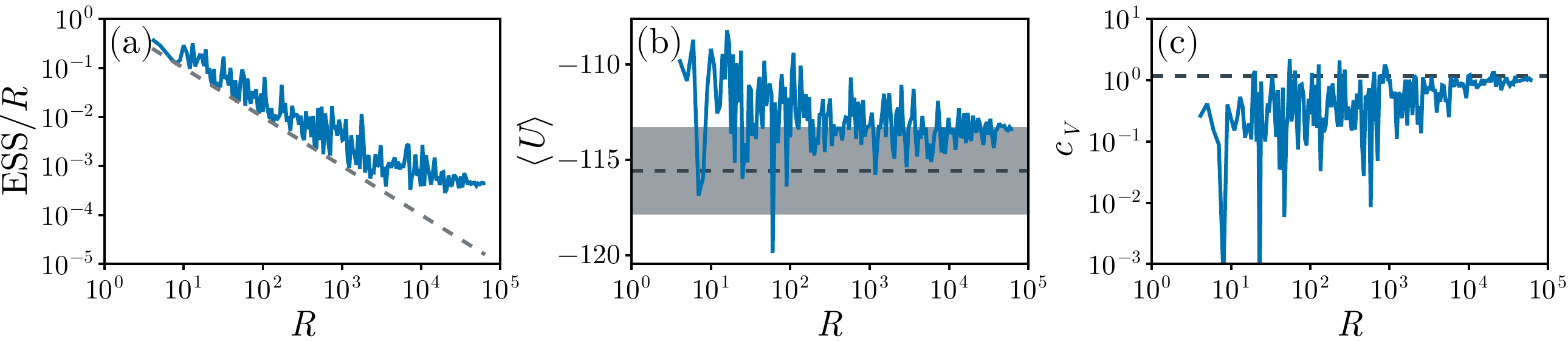}
\caption{Reweighting with filtered samples.
(a) Evolution of the normalized effective sample size $\mathrm{ESS}/R$ as a function of $R$ with energy filtering.
(b) Convergence of the reweighted mean potential energy $\langle U\rangle$ using the filtered samples. The shaded gray band denotes the equilibrium energy fluctuations, exposing a systematic numerical bias.
(c) Convergence of the specific heat $c_V$ as a function of $R$.}
\label{fig:filtering_results}
\end{figure*}

Motivated by this validation, we generate $64\times10^4$ independent samples from \Cref{eq:neuralode} and evaluate the density $\hat{q}(C)$ via backward integration of \Cref{eq:ode_likelihood}, which is approximately 150 times more expensive, only for the $10\%$ lowest-energy configurations. The statistical effect of this selection is illustrated in Fig.~\ref{fig:filtering_results}(a) where the normalized ESS stabilizes near $10^{-3}$ rather than at $10^{-4}$, recall Fig.~\ref{fig:jbb32}(e).

These results demonstrate that pre-filtering configurations using the potential energy before performing the more costly computation of their density is a valid strategy, that allows for a much faster and more efficient convergence of equilibrium measurements via importance sampling. 

\subsubsection{Fundamental limitation of density estimation}

However, this increased statistical resolution reveals that the estimate for $\langle U\rangle$ and $c_V$ is systematically biased. Figures~\ref{fig:filtering_results}(b-c) show that while the plateau values remain within the fluctuations of the target data, the estimators clearly converge to a value that is slightly larger than the true target mean. Since our prior analysis demonstrated that the filtering protocol itself introduces no detectable bias, this systematic shift must represent an intrinsic limitation of the flow architecture that becomes visible only because the statistical variance has been considerably reduced. 

We found that the source of this bias lies in the numerical integration of the flow trajectories needed to estimate densities. In a high-dimensional configuration space, integrating trajectories under a finite numerical tolerance introduces small truncation errors at each step. This numerical drift breaks time-reversibility: comparing the importance weights obtained via forward and backward integration for the same physical configurations reveals a mean relative discrepancy of approximately $10\%$. Crucially, both integration directions are equally affected and one of them cannot be assumed to be better than the other, as both represent distinct numerical approximations of the same underlying continuous flow~\cite{grenioux2026diffusion}. 

The numerical drift appears as the interplay between system complexity and model capacity. Sampling at lower temperatures $T$ requires deeper, heavier networks to capture finer details of the underlying energy landscape.  These highly parameterized velocity fields exhibit strong local curvature and numerical stiffness, making ODE trajectories more prone to discretization errors that accumulate across all dimensions. 
\tcmb{To quantitatively characterize this integration difficulty, we evaluated the maximum singular value (spectral norm) of the Jacobian of the velocity field, $\sigma_{\max}(t) = \norm{\partial\hat v / \partial X}_2$, along the flow trajectories. For the onset temperature $T=1.0$, the maximum spectral norm averages $6.3$ at $t=0.5$. Deep in the supercooled regime at $T=0.32$, this value more than doubles to $13.8$. This larger stiffness directly impacts the integration cost, as the adaptive ODE solver requires $40\%$ more steps at the lower temperature to maintain the same numerical tolerance.}
The numerical errors are therefore larger at lower temperatures, because the integration space is larger. For highly demanding tasks such as sampling glass-forming mixtures, these accumulated errors alter the weight distribution, injecting a systematic bias that compromises the reliability of exact reweighting.

In addition, attempting to suppress these errors by tightening the ODE solver tolerances rapidly becomes unfeasible: evaluating the exact density requires computing the divergence $\nabla \cdot \hat{v}$ at each step via $d N$ backward passes. For stiff flows where the required number of solver steps is already large, tightening tolerances drastically increases the computational cost.
Because this drift cannot be overcome by brute-force integration, it constitutes a fundamental bottleneck for continuous-time Boltzmann generators. We conclude that the strategy will be limited by the large number of parameters needed for larger and more complex systems.

\section{Conclusion and perspectives}

We have developed the Equivariant Riemannian Stochastic Interpolant (ERSI) framework, a generative modeling approach tailored to sample equilibrium configurations of bulk amorphous particle systems with multiple species. In doing so, we establish the exact mathematical requirements under which continuous normalizing flows trained via stochastic interpolants preserve physical symmetries. Specifically, we proved that when the boundary distributions are invariant, restricting the interpolation to a geodesic path guarantees that the optimal time-marginal densities and the optimal velocity field respect the full symmetry group $\invgroup$ of the multi-component mixture on $\mathcal C$. This provides a justification for parameterizing the vector field using a geometry-aware, equivariant graph neural network. 

Benchmark simulations on dense binary and supercooled ternary mixtures demonstrate that explicitly encoding geometric structure and physical symmetries yield more accurate equilibrium sampling and more reliable estimates of physical observables, using fewer generated samples and scaling to larger system sizes than architectures that neglect these properties.

Because Boltzmann generators rely on importance sampling for thermodynamic reweighting, accessing the exact likelihood is essential. In our continuous flow formulation, tracking this density requires integrating a coupled ODE system whose computational cost is much higher than simply generating samples. This presents a significant bottleneck when gathering large sample volumes.

Even more crucially, our work reveals an intrinsic limitation of continuous normalizing flows. Since the adaptive ODE solver is imperfect, local truncation errors accumulate along the continuous trajectories. This breaks time-reversibility and corrupts the importance sampling, which eventually produces biased thermodynamic averages. The limitations becomes more acute for larger models that are needed to generate configurations of larger systems with many-body interactions. Our results therefore reveal a fundamental challenge for continuous-flow generative models in statistical mechanics and call for alternative approaches that preserve exact thermodynamic consistency.

Looking forward, our results point to interesting directions for future research. First, the systematic bias introduced by continuous ODE solvers could be bypassed by adopting discrete-time formulations. Recent developments in diffusion-based sampling~\cite{phillips2024particledenoisingdiffusionsampler, zhang2025efficient} directly discretize the equations of motion to compute the exact likelihood natively through discrete transition kernels, offering a direct route to eliminate integration errors.

Second, our approach assumes access to an initial dataset of biased samples for training the Boltzmann Generator. Relaxing this assumption and developing training strategies that do not rely on such a priori data is an important next step. Promising directions include adaptive MCMCs  \cite{gabrieAdaptiveMonteCarlo2022} possibly combined with sequential tempering \cite{wuSolvingStatisticalMechanics2019,McNaughton2020boosting,bonoPerformanceMachinelearningassistedMonte2025} as proposed for normalizing flows and autoregressive models

\section*{Acknowledgments}

This work was partially performed using HPC resources from GENCI–IDRIS (AD011015234R1, AD011015234R2, AD011015234). G.B. and M.G. acknowledge support from PRAIRIE-PSAI ANR-23-IACL-0008, managed under the France 2030 program. L.B. acknowledges the support of the French Agence Nationale de la Recherche (ANR) under grant ANR-24-CE30-0442 (project GLASSGO). Leonardo G. acknowledges co-funding from the European Union - Next Generation EU. Louis G. acknowledges funding from Hi! Paris. We thank Gerhard Jung for useful discussions, and Jérémy Diharce for his precious involvement in previous versions of this project. 

\section*{Data Availability}

The data that support the findings of this study are openly available~\cite{grenioux_2025_17966995}.

\bibliography{main.bib}

\appendix

\onecolumngrid

\newpage

\section{Target distribution's symmetry group}

\label{app:invariances}

In this section, we establish that the target Boltzmann distribution $\pstar$ is strictly invariant under the group of physical symmetry transformations $\invgroup$ introduced in \Cref{sec:invariances}. We first connect the nearest image distance $\mathrm d_{\torus}$ with the logarithmic map on the torus in \Cref{prop:torus:distance_with_log}. We then establish the $\invgroup$-equivariance of both the logarithmic and exponential maps (\Cref{prop:torus:log_is_log}, \Cref{cor:torus:log_is_log}, and \Cref{prop:torus:M_and_exp}). Finally, \Cref{cor:torus:dist_is_invar} and \Cref{prop:canonical_distribution} synthesize these geometric results to prove our main outcome: that the metric $\mathrm{d}_{\torus}$, the potential energy function $U$, the target density $\pstar$, and the corresponding canonical distribution $\pstar^n$ at fixed composition are all $\invgroup$-invariant.

\begin{definition}
    Let $G$ denote a set of group actions acting on the configuration space $\gC$. A probability density $q$ on $\gC$ is said to be $G$-invariant if, for all $g \in G$ and all $C \in \gC$, $q(g(C)) = q(C)$.
\end{definition}
\begin{lemma}\label{lemma:decomp_g}
    Let $g \in \invgroup$. The action can be decomposed as
    $$
        g : \begin{pmatrix}
            s \\
            X
        \end{pmatrix} \mapsto \begin{pmatrix}
            g^s(C) \\
            g^X(C)
        \end{pmatrix} = \begin{pmatrix}
            A_{\sigma} s \\
            (A_{\sigma} \otimes I_d) h_{M, c, L}(X)
        \end{pmatrix}\eqsp,
    $$
    $$
        h_{M, c, L}(X) = \torusmod{(I_N \otimes M) X + (\mathbf{1}_N \otimes c)}\eqsp,
    $$
    where $\sigma \in S_N$ is a permutation, $A_{\sigma}$ is the associated permutation matrix of size $N \times N$, $M$ is a signed permutation matrix of size $d \times d$ and $c$ is a vector in $\sR^{d}$. The function $h_{M, c, L}$ could be also un-vectorized as
    $$
        h_{M, c, L}(X) = \left\{\tilde{h}_{M,c,L}(\selectparticle{X}{i})\right\}_{i=1}^N, \quad \tilde{h}_{M,c,L}(y) = \torusmod{My + c}\eqsp,
    $$
    which also means that for all $i \in \iinter{1}{N}$
    $$
        \selectparticle{g^X(C)}{i} = \tilde{h}_{M,c,L}(\selectparticle{X}{\sigma(i)})\eqsp.
    $$
\end{lemma}
Note that for any $g \in \invgroup$, $g$ can be written as $g = f_L \circ h_{A,b}$ where
\begin{align}\label{eq:decomp_g}
    f_L : (s, X) \mapsto \begin{pmatrix}
        s \\ X ~\%~ L
    \end{pmatrix} ~~\text{and}~~ h_{A,b} : C \mapsto A C + b\eqsp,
\end{align}
for $A$ an orthogonal matrix and $b$ a vector in $\sR^{N\abs{\gS} + Nd}$ which are defined in \Cref{lemma:decomp_g}.
The set of actions induced by $h_{M, c, L}$ on $\torus^{N}$ is denoted $\gG_{\torus^N}$ and the one induced by $\tilde{h}_{M,c,L}$ on $\torus$ is denoted $\gG_{\torus}$. From \Cref{lemma:decomp_g}, we also deduce that $A$ and $b$ in \Cref{eq:decomp_g} can be written as
\begin{align}\label{eq:A_in_g_decomp}
    A = \begin{pmatrix}
        A_{\sigma} & 0 \\
        0 & (A_{\sigma} \otimes I_d)  (I_N \otimes M)
    \end{pmatrix}, \quad b = \mathbf{1}_N \otimes c\eqsp.
\end{align}

\begin{lemma}\label{lemma:torus:torus_decomp}
    Let $X, Y \in \sR^{Nd}$. We have
    \begin{equation}
        \torusmod{\left(X ~\%~ L\right) + Y} = \torusmod{X + Y}\eqsp .
    \end{equation}
\end{lemma}
\begin{proof}
    We have that $X = (X ~\%~ L) + k_X L$ where $k_X \in \mathbb{Z}^{Nd}$ which leads to
    $$
        \torusmod{\left(X ~\%~ L\right) + Y} = \torusmod{X + Y - k_X} = \torusmod{X + Y}\eqsp.
    $$
\end{proof}

\begin{lemma}\label{lemma:torus:wierd_mod_lemma}
    Let $y \in \mathbb{R}$ and $\epsilon \in \{-1, +1\}$, then $\torusmod{\epsilon y} = \torusmod{\epsilon \left(y ~\%~ L\right)}$.
\end{lemma}
\begin{proof}
    If $\epsilon = 1$ this is obviously true. If $\epsilon = -1$, we are trying to prove that
    $$
        \torusmod{-y} = \torusmod{- \left(y ~\%~ L\right)}\eqsp.
    $$
    By definition, we have that
    \begin{align*}
        \torusmod{-y} &= -y - \left\lfloor\frac{-y}{L}\right\rfloor L\eqsp, \\
        \torusmod{- \left(y ~\%~ L\right)} &= -y - \left(\left\lfloor \frac{-y}{L} + k\right\rfloor - k\right) L\eqsp,
    \end{align*}
    where $k = \lfloor y / L\rfloor$. Using the fact that for any $x \in \mathbb{R}$ and $n \in \mathbb{N}$,
    $$
        \lfloor x + n \rfloor - n = \lfloor x \rfloor\eqsp,
    $$
    we get the intended result.
\end{proof}

\begin{lemma}\label{lemma:M_is_torus_equi}
    Let $M \in \gB_d$ and $X \in \sR^d$, then $\torusmod{MX} = \torusmod{M \left[X ~\%~ L\right]}$.
\end{lemma}
\begin{proof}
    For any $M \in \gB_d$, there exist $\epsilon \in \{-1,1\}^d$ and $\sigma \in \mathrm{S}_d$ such that for all $i,j \in \iinter{1}{d}$, $M_{i,j} = \epsilon(i)\delta_{\sigma(i),j}$. For all $i \in \iinter{1}{d}$, we have
    \begin{align*}
        \left(\torusmod{MX}\right)_i = \torusmod{\epsilon(i) X_{\sigma(i)}}, \quad \left( M \left[X ~\%~ L\right]\right)_i = \epsilon(i) \left[X ~\%~ L\right]_{\sigma(i)}\eqsp.
    \end{align*}
    Use \Cref{lemma:torus:wierd_mod_lemma} on each coordinate to conclude.
\end{proof}

\begin{proposition}\label{prop:torus:log_as_floor}
    The logarithmic map can be written for any $X\in \torus$ and $Y \in \tangenttorus{X}$
    $$
    $$
    $$
        \log_X(Y) = \frac{L}{2\pi}\,\atantwo\left(\sin\left(\frac{2\pi}{L}(X-Y)\right), \,\cos\left(\frac{2\pi}{L}(X-Y)\right)\right)\eqsp.
    $$
\end{proposition}
\begin{proof}
    For any $X\in \torus$ and $Y \in \tangenttorus{X}$, we have that
    \begin{align*}
        \log_X(Y) & =\torusmod{Y - X + \frac{L}{2}} - \frac{L}{2}\\
        &= Y-X-L\left\lfloor\frac{Y-X+\frac{L}{2}}{L}\right\rfloor\eqsp,\\
        &=\frac{L}{\pi}\arctan\left(\tan\left(\frac{\pi}{L}\left(Y-X\right)\right)\right)\eqsp, & \left(\text{Using }\arctan\tan\alpha=\alpha-\pi\left\lfloor\frac{\alpha}{\pi}+\frac{1}{2}\right\rfloor\right)\\
        & =\frac{L}{\pi}\arctan\left(\frac{\sin\left(\frac{2\pi}{L}\left(Y-X\right)\right)}{1+\cos\left(\frac{2\pi}{L}\left(Y-X\right)\right)}\right)\eqsp, & \left(\text{Using }\tan\alpha=\frac{\sin2\alpha}{1+\cos2\alpha}\right)\eqsp\\
        &=\frac{L}{2\pi}\,\atantwo\left(\sin\left(\frac{2\pi}{L}(X-Y)\right), \,\cos\left(\frac{2\pi}{L}(X-Y)\right)\right)\eqsp.
    \end{align*}
\end{proof}

\begin{lemma}\label{lemma:torus:argmin_and_floor}
    For all $a \in \mathbb{R}$, $-\lfloor\frac{a}{L} + \frac{1}{2}\rfloor = \argmin_{k \in \mathbb{Z}} \abs{a + kL}$ holds.
\end{lemma}
\begin{proof}
    By definition, we have that
    $$
        \left\lfloor\frac{a}{L}\right\rfloor L \leq a \leq \left(\left\lfloor\frac{a}{L}\right\rfloor + 1\right)L\eqsp,
    $$
    which means that
    $$
        k^{\star} = \argmin_{k \in \mathbb{Z}} \abs{a + kL} \in \left\{-\left\lfloor\frac{a}{L}\right\rfloor, -\left(\left\lfloor\frac{a}{L}\right\rfloor +1\right)\right\}\eqsp.
    $$
    \underline{Case 1}: $k^{\star} = -\lfloor a / L\rfloor$
    This means that $a / L$ is closer to $\lfloor a/L\rfloor$ than to $\lfloor a / L \rfloor + 1$ which implies that $a/L + 1/2$ is also closer to $\lfloor a / L\rfloor$ which means that
    $$
        \left\lfloor\frac{a}{L}\right\rfloor = \left\lfloor\frac{a}{L} + \frac{1}{2}\right\rfloor\eqsp,
    $$
    and that $k^{\star} = -\lfloor a/L + 1/2\rfloor$.

    \underline{Case 2}: $k^{\star} = -(\lfloor a / L\rfloor + 1)$
    This means that $a / L$ is closer to $\lfloor a/ L \rfloor + 1$ than to $\lfloor a / L\rfloor$ which implies that $a / L + 1/2$ is also closer to $\lfloor a / L\rfloor + 1$ which means that
    $$
        \left\lfloor\frac{a}{L}\right\rfloor + 1= \left\lfloor\frac{a}{L} + \frac{1}{2}\right\rfloor\eqsp,
    $$
    and that $k^{\star} = -\lfloor a / L + 1 /2\rfloor$.
\end{proof}

\begin{lemma}\label{lemma:torus:min_decomp}
    Let $f : \mathbb{R}^k \to \mathbb{R}^+$ be defined for all $x \in \mathbb{R}^k$ as
    $$
        f(x) = \sum_{i=1}^k f_i(x_i), \quad f_i : \mathbb{R} \to \mathbb{R}^+\eqsp.
    $$
    Then the minimum of $f$ decomposes as
    $$
        \argmin_{x \in \mathbb{R}^k} f(x) = \argmin_{x_1 \in \mathbb{R}} f_1(x) \times \ldots \times \argmin_{x_k \in \mathbb{R}} f_k(x)\eqsp.
    $$
\end{lemma}
\begin{proof}
    Suppose that there exist $\bar{x} \in \argmin_{x \in \mathbb{R}^k} f(x)$ and $\bar{x} \notin \argmin_{x_1 \in \mathbb{R}} f_1(x) \times \ldots \times \argmin_{x_k \in \mathbb{R}} f_k(x)$. Then there exists $i \in \{1, \ldots, k\}$ such that $\bar{x}_i \notin \argmin_{x_i \in \mathbb{R}} f_i(x)$. Let $\hat{x}_i \in \argmin_{x_i \in \mathbb{R}} f_i(x)$, then $\hat{x} \in \mathbb{R}^k$ defined as
    $$
        \hat{x}_j = \begin{cases}
            \bar{x}_j & \text{ if } j \neq i \\
            \hat{x}_i & \text{ otherwise }
        \end{cases}.
    $$
    Then $f(\hat{x}) \leq f(\bar{x})$ which contradicts the assumptions. This gives the left-right inclusion, and the right-left one is trivial.
\end{proof}

\begin{proposition}\label{prop:torus:distance_with_log}
    Let $X, Y \in \torus$, the pre-metric $\mathrm{d}_{\torus}$ can be written as
    $$
        \torusdist{X}{Y} = \norm{\log_X(Y)},\quad \text{ for any } X, Y \in \torus\eqsp.
    $$
\end{proposition}
\begin{proof}
    We have that $\log_X(Y) = Y - X + k_{X,Y}L$ where
    $$
        k_{X,Y} = -\left\lfloor\frac{Y - X + \frac{L}{2}}{L}\right\rfloor = -\left\lfloor\frac{Y - X}{L} + \frac{1}{2}\right\rfloor.
    $$
    Moreover, we have that
    \begin{align*}
        \torusdist{X}{Y} &= \min_{k \in \mathbb{Z}^{d}} \norm{X - Y + kL}
        = \min_{k \in \mathbb{Z}^{d}} \norm{Y - X + kL}
        = \min_{k \in \mathbb{Z}^{d}} \norm{Y - X + kL}^2\eqsp,
    \end{align*}
    which implies that
    $$
        \torusdist{X}{Y} = \min_{k \in \mathbb{Z}^{d}} \sum_{i=1}^d (Y_i - X_i + k_iL)^2\eqsp.
    $$
    Using \Cref{lemma:torus:argmin_and_floor} and the monotonicity of the square root on $\mathbb{R}^+$, we have that that for all $i \in \iinter{1}{d}$
    $$
       \argmin_{k \in \mathbb{Z}} (Y_i - X_i + k L)^2 = [k_{X,Y}]_i\eqsp.
    $$
    By \Cref{lemma:torus:min_decomp}, this implies that
    $$
        \argmin_{k \in \mathbb{Z}^{d}} \sum_{i=1}^d (Y_i - X_i + k_iL)^2 = k_{X,Y}\eqsp,
    $$
    leading to
    $$
        \mathrm d_\torus(X,Y) = \norm{Y - X + k_{X,Y}L} = \norm{\log_X(Y)}\eqsp.
    $$
\end{proof}

\begin{proposition}\label{prop:torus:dist}
    Let $X, Y \in \torus$, a pre-metric (see \cite[Section 3.2]{chen2024flow} for the definition) between $X$ and $Y$ in $\torus$ can be defined as
    \begin{equation}
        \torusdist{X}{Y} = \min_{k \in \mathbb{Z}^{d}} d^E(X, Y + kL)\eqsp,
    \end{equation}
    where $d^E(X,Y) = \norm{X - Y}$ is the euclidean distance.
\end{proposition}
\begin{proof}
    Let $X, Y \in \torus$, we obviously have that $\torusdist{X}{Y} \geq 0$ checking the non-negativity. Moreover, using the fact that $d^E$ is a distance
    $$
        \torusdist{X}{Y} = 0 \iff \; \exists k \in \mathbb{Z}^{d}, d^E(X, Y + kL) = 0 \iff X = Y ~\%~ L \iff X = Y\eqsp,
    $$
    which checks positivity. Using \Cref{prop:torus:distance_with_log}, we have that $\nabla \torusdist{X}{Y} = \log_XY$, which leads to
    \begin{align*}
        \nabla \torusdist{X}{Y} \neq 0 \iff \log_XY \neq 0 & \iff \torusmod{Y - X + \frac{L}{2}} - \frac{L}{2} \neq 0 \\
        & \iff \exists ! k \in \mathbb{Z}^{Nd}, ~Y - X + \frac{L}{2} + kL - \frac{L}{2} \neq 0 \\
        & \iff \torusmod{Y - X} \neq 0 \\
        & \iff X \neq Y\eqsp.
    \end{align*}
\end{proof}

\begin{proposition}\label{prop:torus:M_and_exp}
    Let $A, B \in \mathbb{R}^{d}$, $M \in B_d$ and $u \in \mathbb{R}^{d}$ then
    $$
        \exp_{\torusmod{MA+u}}\left(MB\right) = \torusmod{M \exp_AB + u}
    $$
    holds.
\end{proposition}
\begin{proof}
    Using \Cref{lemma:torus:torus_decomp}, we have that
    $$
        \exp_{\torusmod{MA+u}}\left(MB\right) = \torusmod{\exp_{\torusmod{MA}} MB + u}\eqsp.
    $$
    Moreover, given that the signed permutation matrix $M$ is characterized by permutation $\sigma$ and a vector $\epsilon \in \{-1, 1\}^d$, we have for all $i \in \iinter{1}{d}$
    $$
        \left[M(A+B)\right]_i = \epsilon_i (A_{\sigma(i)} + B_{\sigma(i)})\eqsp.
    $$
    Using Lemmas \ref{lemma:torus:torus_decomp} and \ref{lemma:torus:wierd_mod_lemma}, we get that
    $$
        \left[\exp_{\torusmod{MA}} MB\right]_i = \torusmod{\epsilon_i (A_{\sigma(i)} + B_{\sigma(i)})} = \torusmod{\epsilon_i \left[\torusmod{A_{\sigma(i)} + B_{\sigma(i)}}\right]}\eqsp.
    $$
    Additionally, note that
    $$
        \left[\torusmod{M \exp_A B}\right]_i = \torusmod{\epsilon_i \left[\torusmod{A_{\sigma(i)} + B_{\sigma(i)}}\right]}\eqsp,
    $$
    which proves that $\exp_{\torusmod{MA}} MB = \torusmod{M \exp_A B}$. If we put this expression in the first equation and use \Cref{lemma:torus:torus_decomp}, we get
    $$
        \exp_{\torusmod{MA+u}}\left(MB\right) = \torusmod{\torusmod{M \exp_A B} + u} = \torusmod{M \exp_A B + u}\eqsp.
    $$
\end{proof}

\begin{proposition}\label{prop:torus:log_is_log}
    Let $X, Y \in \torus$. For any $M \in B_d$,
    $$
        \log_{\torusmod{MX}}(\torusmod{MY}) = M \log_X(Y)\eqsp,
    $$
    holds.
\end{proposition}
\begin{proof}
    Let $x, y \in [0,L]$, we have that
    $$
        \log_{\torusmod{-x}}(\torusmod{-y}) = \frac{L}{\pi}\arctan\left(\frac{\sin\left(\frac{2\pi}{L}\left(\torusmod{-y}-\torusmod{-x}\right)\right)}{1+\cos\left(\frac{2\pi}{L}\left(\torusmod{-y}-\torusmod{-x}\right)\right)}\right).
    $$
    By definition, there exists $k, k' \in \mathbb{Z}$ such that
    $$
        \torusmod{-x} = -x + kL, \quad \text{ and } \torusmod{-y} = -y + k'L\eqsp.
    $$
    This leads to
    \begin{align*}
        \frac{2\pi}{L}\left(\torusmod{-y}-\torusmod{-x}\right) &= \frac{2\pi}{L}\left(\left[ -y + k'L\right]-\left[-x + kL\right]\right)\eqsp, \\
        &= -\frac{2\pi}{L} (y - x) + 2\pi \underbrace{(k - k')}_{\in \mathbb{Z}}\eqsp. 
    \end{align*}
    Using the periodicity of the sine and cosine in the above formula, we get
    $$
        \log_{\torusmod{-x}}(\torusmod{-y}) = \frac{L}{\pi}\arctan\left(\frac{-\sin\left(\frac{2\pi}{L}\left(y-x\right)\right)}{1+\cos\left(\frac{2\pi}{L}\left(y-x\right)\right)}\right) = -\log_x(y).
    $$
    Let $M \in B_d$ and $X, Y \in \torus$. By definition, for all $i \in \iinter{1}{d}$, $(MX)_i = \epsilon_i X_{\sigma(i)}$. Moreover, using the previous result
    \begin{align*}
        &\left[\log_{\torusmod{MX}}(\torusmod{MY})\right]_i \\
        &= \log_{[\torusmod{MX}]_i}([\torusmod{MY}]_i) = \log_{\torusmod{[MX]_i}}(\torusmod{[MY]_i})\eqsp, \\
        &= \log_{\torusmod{\epsilon_i X_{\sigma(i)}}}(\torusmod{\epsilon_i Y_{\sigma(i)}}) = \epsilon_i \log_{X_{\sigma(i)}}(Y_{\sigma(i)})\eqsp, \\
        &= [M \log_X(Y)]_i\eqsp,
    \end{align*}
    which concludes the proof.
\end{proof}

\begin{corollary}\label{cor:torus:log_is_log}
    Let $X, Y \in \torus$. For any $M \in B_d$ and $u \in \mathbb{R}^d$,
    \begin{equation}
        \log_{\torusmod{MX + u}}(\torusmod{MY + u}) = M \log_X(Y).
    \end{equation}
\end{corollary}
\begin{proof}
    Using \Cref{lemma:torus:torus_decomp} applied twice, we have
    \begin{align*}
        \log_{\torusmod{X + u}}(\torusmod{Y + u}) &= \torusmod{\torusmod{Y + u} - \torusmod{Y + u} + \frac{L}{2}} - \frac{L}{2}\eqsp, \\
        &= \torusmod{Y + u - Y - u + \frac{L}{2}} - \frac{L}{2}\eqsp, \\
        &= \log_{X ~\%~ L}(Y ~\%~ L)\eqsp.
    \end{align*}
    The result comes by applying this remark in \Cref{prop:torus:log_is_log} with $MX \in \torus$ and $MY \in \torus$.
\end{proof}

\begin{corollary}\label{cor:torus:dist_is_invar}
    The pre-metric $\mathrm{d}_{\torus}$ is $\mathcal{G}_{\torus}$-invariant, i.e, for all $X, Y \in \torus$
    $$
        \torusdist{g(X)}{g(Y))} = \torusdist{X}{Y}, \quad \text{ for all } g \in \mathcal{G}_{\torus}\eqsp.
    $$
    Consequently, the potential $U$ (\ref{eq:potential}) and the induced density $\pstar$ (\ref{eq:boltzmann}) are both $\invgroup$-invariant.
\end{corollary}
\begin{proof}
    This comes from the fact that, according to \Cref{prop:torus:distance_with_log}, the distance can be written as the norm of the logarithmic map, which itself is equivariant as per \Cref{cor:torus:log_is_log} and using the decomposition of the action in \Cref{lemma:decomp_g}. Using the fact that the scaling matrix is always orthogonal, we get the invariances with respect to symmetries and translations. For the permutation invariance, we just use the permutation invariance of the sum.
\end{proof}

\begin{corollary}
\label{prop:canonical_distribution}
    Given a fixed composition $\{n_s\}_{s \in \gS}$, the canonical distribution associated to $\pstar$ is
    $$
        \pstar^{n}(\rmd s, \rmd \text{vol}_X) \;=\; \frac{1}{\mathcal{Z}^n}\exp\left(-\frac{U(s,X)}{k_{\rmB}T}\right) 
    \times \prod_{s' \in \gS} \delta_{\sum_{i=1}^N \delta_{s'}(s_i)}(n_{s'}) \,\rmd s \,\rmd \text{vol}_X\eqsp.
    $$
    This density is $\invgroup$-invariant.
\end{corollary}
\begin{proof}
    Use the $\invgroup$-invariance of $\pstar$ and the permutation invariance of the sum.
\end{proof}

\section{Equivariant Riemannian Stochastic Interpolants}

\label{app:ersi}

This section provides the formal theoretical guarantees for the Equivariant Riemannian Stochastic Interpolants (ERSI) framework introduced in Section~\ref{sec:ersi}, establishing end-to-end equivariance from the underlying stochastic process down to the neural velocity field architecture. 
First, Proposition~\ref{prop:eq_interpolation} demonstrates that an equivariant interpolant paired with an invariant base distribution produces invariant marginal distributions across all $t \in [0, 1]$, while Proposition~\ref{prop:optimal_velocity_equivariance} proves that this construction yields an equivariant optimal velocity field. 
Proposition~\ref{prop:geodesic_interpolant_is_equi} then verifies that our specific geodesic interpolation defined in Eq.~\eqref{eq:geodesic_interpolant} satisfies this condition. 
Next, we generalize the results of \cite{kohler2020equivariant} to our non-linear symmetry group $\invgroup$ in Propositions~\ref{prop:eq_transport_map} and~\ref{prop:equivariant_field_to_equivariant_map}. 
Finally, Proposition~\ref{prop:egnn_is_equi} proves by induction that our Equivariant Graph Neural Network architecture (Section~\ref{sec:equi_model}) is $\invgroup$-equivariant at every layer, closing the gap between our theoretical framework and its practical implementation.

\begin{lemma}\label{lemma:g_c_is_group}
    $\invgroup$ is a group for the composition operation.
\end{lemma}

\begin{proof}
    The identity is in $\invgroup$ as it is a special case of the invariances. Let $g^1, g^2 \in \invgroup$ we can decompose them as $g^1 = f_L \circ h_{A^1, b^1}$ and $g^2 = f_L \circ h_{A^2, b^2}$ where $A_1$ (respectively $A^2$) can be decomposed into $A^{1}_{\sigma}$, $M^1$ (respectively $A^{2}_{\sigma}$, $M^2$) and $b^1$ (respectively $b^2$) into $b^1 = \mathbf{1}_N \otimes c^1$ (respectively $b^2 = \mathbf{1}_N \otimes c^2$). Using both Lemmas \ref{lemma:M_is_torus_equi} and \ref{lemma:torus:torus_decomp}, $g^1 \circ g^2$ can be written for any $C \in \gC$
    $$
        (g^1 \circ g^2)(C) = f_L \left[\begin{pmatrix}
        A_{\sigma^1} A_{\sigma^2} & 0 \\
        0 & (A_{\sigma^1} A_{\sigma^2} \otimes I_d)  (I_N \otimes M^1 M^2)
    \end{pmatrix} C + (\mathbf{1}_N \otimes (c^1 + M^1 c^2))\right]\eqsp,
    $$
    which shows, by \Cref{lemma:decomp_g}, that $g^1 \circ g^2 \in \invgroup$. Finally, let $g \in \invgroup$ be written under the same decomposition, using Lemmas \ref{lemma:M_is_torus_equi} and \ref{lemma:torus:torus_decomp} again, we have for any $C \in \gC$
    \begin{align*}
        g^{-1}(C) &= f_L \left[\begin{pmatrix}
        A_{\sigma^{-1}} & 0 \\
        0 & (I_N \otimes M^{-1}) (A_{\sigma^{-1}} \otimes I_d)
    \end{pmatrix} C - (\mathbf{1}_N \otimes c)\right]\eqsp,\\
    &= f_L \left[\begin{pmatrix}
        A_{\sigma^{-1}} & 0 \\
        0 & (A_{\sigma^{-1}} \otimes I_d) (I_N \otimes M^{-1})
    \end{pmatrix} C + (\mathbf{1}_N \otimes (-c))\right]\eqsp.
    \end{align*}
    which shows, by \Cref{lemma:decomp_g}, that $g^{-1} \in \invgroup$.
\end{proof}

\begin{lemma}\label{lemma:jacobian_g}
    For all $g = f_L \circ h_{A,b}\in \invgroup$,  as in~\Cref{eq:decomp_g}, let $J_g$ be the Jacobian of $g$. Then $J_g(C)=A$, for almost every $C\in\gC$, and $\abs{J_g(C)} = 1$ almost everywhere.
\end{lemma}
\begin{proof}
    Let $C\in\gC$, then by the chain rule
    $$
    J_g(C)=J_{f_L}\left(h_{A,b}(C)\right)J_{h_{A,b}}(C)\eqsp.
    $$
    The modulo has unit derivative (except from jump points), so $J_{f_L}=I_{(N+1)d}$ almost everywhere. Since $h_{A,b}$ is affine, $J_{h_{A,b}}=A$, so $J_g(C)=A$, for almost every $C\in\gC$. Moreover, $\abs{J_g(C)}=\abs{A}$. Using \Cref{eq:A_in_g_decomp}, we can show that $A \in \gB_{Nd}$ which implies that $\abs{A} = 1$.
\end{proof}

\begin{lemma}\label{lemma:delta_group}
    Let $g\in\invgroup$ and $C, D \in\mathcal C$ and then 
    $$
    \int_{\mathcal C}\varphi(C)\delta_{g(D)}\left(g(C)\right)\rmd C=\int_{\mathcal C}\varphi(C)\delta_D\left(C\right)\rmd C\eqsp,
    $$
    for all $\varphi\in C^{\infty}(\mathcal C).$
\end{lemma}
\begin{proof}
    Make the change of variable $C=g^{-1}(U)$, since $g$ is a group element,
    $$
        U = g(C),\quad \rmd C = \abs{J_{g^{-1}}(U)}\rmd U\eqsp.
    $$
    Using~\Cref{lemma:jacobian_g}, $\abs{J_{g^{-1}}(U)}=1$, thus
    \begin{align*}
    \int_{\mathcal C}\varphi(C)\delta_{g(D)}\left(g(C)\right)\rmd C&=\int_{\mathcal C}\varphi\left(g^{-1}(U)\right)\delta_{g(D)}\left(U\right)\rmd U = \varphi(D) = \int_{\mathcal C}\varphi(C)\delta_{D}(C)\rmd C \eqsp. 
    \end{align*}
\end{proof}

\begin{definition}
    An interpolation function $I : [0,1] \times \gC \times \gC \to \gC$ is said to be $\invgroup$- equivariant if for all $g \in \invgroup$, $I(t,g(C_0),g(C_1))= g(I(t,C_0,C_1))$ holds for any $t \in [0,1]$ and $C_0, C_1 \in \gC$.
\end{definition}

\begin{proposition}
\label{prop:eq_interpolation}
    Given a base $p_{\mathrm{base}}$ and the target $\pstar$ distribution both $\invgroup$-invariant and a $\invgroup$-equivariant interpolant $I$, then the induced marginal densities $(p_t)_{t=0}^1$ are all $\invgroup$-invariant.
\end{proposition}
\begin{proof}
    The marginal densities are defined as
    $$
    p_t(C) = \int_{\mathcal C}\delta_{I\left(t,C_0,C_1\right)}\left(C\right)\pbase\left(C_0\right)\pstar\left(C_1\right)\rmd C_0\rmd C_1 \eqsp . 
    $$
    Let $g\in\invgroup$, then
    $$
        p_t(g(C)) = \int_{\mathcal C}\delta_{I\left(t,C_0,C_1\right)}\left(g(C)\right)\pbase\left(C_0\right)\pstar\left(C_1\right)\rmd C_0\rmd C_1 \eqsp .
    $$
    Make the change of variable $C_0 = g(U_0)$, $C_1 = g(U_1)$, using \Cref{lemma:jacobian_g} we get
    \begin{align}
         p_t(g(C)) &= \int_{\mathcal C}\delta_{I\left(t,g\left(U_0\right), g\left(U_1\right)\right)}\left(g(C)\right)\pbase\left(g\left(U_0\right)\right)\pstar\left(g\left(U_1\right)\right)\rmd U_0\rmd U_1\eqsp, \nonumber \\
         &=\int_{\mathcal C}\delta_{I\left(t,g\left(U_0\right), g\left(U_1\right)\right)}\left(g(C)\right)\pbase\left(U_0\right)\pstar\left(U_1\right)\rmd U_0\rmd U_1\eqsp, \label{eq:proof:eq_interpolation:l1}\\
         & = \int_{\mathcal C}\delta_{g\left(I\left(t,U_0, U_1\right)\right)}\left(g(C)\right)\pbase\left(U_0\right)\pstar\left(U_1\right)\rmd U_0\rmd U_1\eqsp, \label{eq:proof:eq_interpolation:l2} \\
         & = \int_{\mathcal C}\delta_{I\left(t,U_0, U_1\right)}\left(C\right)\pbase\left(U_0\right)\pstar\left(U_1\right)\rmd U_0\rmd U_1\eqsp, \label{eq:proof:eq_interpolation:l3} \\
         & = p_t(C)\eqsp. \nonumber
    \end{align}
    We used the invariance of $\pbase$ and $\pstar$ to get (\ref{eq:proof:eq_interpolation:l1}), the equivariance of $I$ to go from (\ref{eq:proof:eq_interpolation:l1}) to (\ref{eq:proof:eq_interpolation:l2}) and \Cref{lemma:delta_group} to go from (\ref{eq:proof:eq_interpolation:l2}) to (\ref{eq:proof:eq_interpolation:l3}).
\end{proof}

\begin{lemma}\label{lemma:equi_derivative}
    Let $g\in\invgroup$ and $I:[0,1]\times\gC\times\gC\to\gC$ be a differentiable, $\invgroup$-equivariant interpolation function. Then, the time derivative $\partial_t I$ is $\gC$-equivariant, i.e., it verifies for any $t \in [0, 1]$, $C_0, C_1 \in \gC$
    $$
        \partial_t I(t, g(C_0), g(C_1)) = A \partial_t I(t, C_0, C_1), \quad \text{ for all } g \in \invgroup \text{ such that } g = f_L \circ h_{A,b}\eqsp.
    $$
\end{lemma}
\begin{proof}
    Let $C_0,C_1\in\gC$, then
    \begin{align*}
    \partial_t I\left(t,g\left(C_0\right), g\left(C_1\right)\right)&=\partial_tg\left(I\left(t,C_0,C_1\right)\right) & & \left(\text{Equivariance of } I\right)\\
    & = J_g\left(I\left(t,C_0,C_1\right)\right)\partial_tI\left(t,C_0,C_1\right) & & \left(\text{Chain rule} \right)\\
    &=A\partial_tI\left(t,C_0,C_1\right)& & \left(\text{\Cref{lemma:jacobian_g}} \right),
    \end{align*}
    where $A$ is defined in~\Cref{eq:decomp_g}.
\end{proof}

\begin{proposition}
    \label{prop:optimal_velocity_equivariance}
    Given a $\invgroup$-equivariant interpolation function $I$, if $\pbase$ and $\pstar$ are $\invgroup$-invariant, then the corresponding optimal velocity field is $\invgroup$-equivariant.
\end{proposition}
\begin{proof}
    The optimal velocity field is defined as
    $$
        u^{\star}(t,C)=\frac{1}{p_t(C)}\int_{\mathcal C}\partial_tI\left(t,C_0,C_1\right)\delta_{I\left(t,C_0,C_1\right)}(C)\pbase\left(C_0\right)\pstar\left(C_1\right)\rmd C_0\rmd C_1\eqsp.
    $$
    Since $p_t$ is $\invgroup$-invariant (see~\Cref{prop:eq_interpolation}), we need to show that the integral
    $$
        \mathcal I(C)=\int_{\mathcal C}\partial_tI\left(t,C_0,C_1\right)\delta_{I\left(t,C_0,C_1\right)}(C)\pbase\left(C_0\right)\pstar\left(C_1\right)\rmd C_0\rmd C_1
    $$
    is a $\invgroup$-equivariant vector field. Let $g\in\invgroup$, 
    $$
        \mathcal I\left(g(C)\right)=\int_{\mathcal C}\partial_tI\left(t,C_0,C_1\right)\delta_{I\left(t,C_0,C_1\right)}\left(g(C)\right)\pbase\left(C_0\right)\pstar\left(C_1\right)\rmd C_0\rmd C_1\eqsp.
    $$
    Make the change of variable $C_0=g\left(U_0\right)$, $C_1=g\left(U_1\right)$, using \Cref{lemma:jacobian_g},
    \begin{align}
        \mathcal I\left(g(C)\right)&=\int_{\mathcal C}\partial_tI\left(t,g\left(U_0\right),g\left(U_1\right)\right)\delta_{I\left(t,g\left(U_0\right),g\left(U_1\right)\right)}\left(g(C)\right)\pbase\left(g\left(U_0\right)\right)\pstar\left(g\left(U_1\right)\right)\rmd U_0\rmd U_1 \nonumber \\
        &=\int_{\mathcal C}\partial_tI\left(t,g\left(U_0\right),g\left(U_1\right)\right)\delta_{I\left(t,g\left(U_0\right),g\left(U_1\right)\right)}\left(g(C)\right)\pbase\left(U_0\right)\pstar\left(U_1\right)\rmd U_0\rmd U_1 \label{eq:eq_optimal_velocity:eq2}\\
        &=\int_{\mathcal C}A\partial_tI\left(t,U_0,U_1\right)\delta_{g\left(I\left(t,U_0,U_1\right)\right)}\left(g(C)\right)\pbase\left(U_0\right)\pstar\left(U_1\right)\rmd U_0\rmd U_1 \label{eq:eq_optimal_velocity:eq3} \\
        &=\int_{\mathcal C}A\partial_tI\left(t,U_0,U_1\right)\delta_{I\left(t,U_0,U_1\right)}\left(C\right)\pbase\left(U_0\right)\pstar\left(U_1\right)\rmd U_0\rmd U_1 \label{eq:eq_optimal_velocity:eq4} \\
        &=A\mathcal I(C)\eqsp, \nonumber
    \end{align}
    where we use the invariance of $\pbase$ and $\pstar$ in (\ref{eq:eq_optimal_velocity:eq2}), then the equivariance of $\partial_t I$ due to \Cref{lemma:equi_derivative} to go from (\ref{eq:eq_optimal_velocity:eq2}) to (\ref{eq:eq_optimal_velocity:eq3}) and \Cref{lemma:delta_group} to go from (\ref{eq:eq_optimal_velocity:eq3}) to (\ref{eq:eq_optimal_velocity:eq4}).
\end{proof}

\begin{proposition}\label{prop:geodesic_interpolant_is_equi}
    The interpolant in~\Cref{eq:geodesic_interpolant} is $\invgroup$-equivariant.
\end{proposition}
\begin{proof}
    Let $g \in \invgroup$. Consider the decomposition of $g$ introduced in \Cref{lemma:decomp_g}. If we denote
    $$
        I_L : (t, C_0, C_1) \mapsto \begin{pmatrix}
            I_L^s(t, C_0, C_1) \\
            I_L^X(t, C_0, C_1)
        \end{pmatrix}\eqsp,
    $$
    then
    $$
        I_L^s(t, g(C_0), g(C_1)) = (1-t) A_{\sigma} s_0 + t A_{\sigma} s_1 = A_{\sigma} I_L^s(t, C_0, C_1)\eqsp.
    $$
    Moreover, for every $i \in \iinter{1}{N}$, by definition
    \begin{align*}
        \selectparticle{I_L^X(t, g(C_0), g(C_1))}{i} &= \exp_{\tilde{h}_{M,c,L}(\selectparticle{(X_0)}{\sigma(i)})} \left(t \log_{\tilde{h}_{M,c,L}(\selectparticle{(X_0)}{\sigma(i)})} \tilde{h}_{M,c,L}(\selectparticle{(X_1)}{\sigma(i)})\right)\eqsp,\\
        &= \exp_{\tilde{h}_{M,c,L}(\selectparticle{(X_0)}{\sigma(i)})} \left(t M \log_{\selectparticle{(X_0)}{\sigma(i)}} \selectparticle{(X_1)}{\sigma(i)}\right)\eqsp,\\
        &= \torusmod{M\exp_{\selectparticle{(X_0)}{\sigma(i)}}\left(t \log_{\selectparticle{(X_0)}{\sigma(i)}} \selectparticle{(X_1)}{\sigma(i)}\right) + c} \\
        &= \selectparticle{\tilde{h}_{M,c,L}(I^X(t, C_0, C_1))}{\sigma(i)}
    \end{align*}
    where we used successively \Cref{cor:torus:log_is_log} and \Cref{prop:torus:M_and_exp}, which concludes the proof.
\end{proof}

We consider a velocity field $\hat{v}$ on $\gC$ and denote by $\hat T_t$ the map transporting an initial configuration of particles along the velocity field between time $0$ and time $t$. In order words, $\hat T_t(c_0)$ gives the solution of the ODE $\rmd C_u = \hat{v}(u, C_u) \rmd u$ at time $t$ for the initial condition $C_0 = c_0$. Given a base distribution $\pbase$, we seek to learn $\hat v $ such that the push-forward of $\pbase$ through the transport map $\hat{T}_t$, which is the marginal distribution of the RSI model denoted by $\hat q_t$ below, is $\invgroup$-invariant. For this, we rely on \Cref{prop:eq_transport_map} extending the result of \cite[Theorem 1]{kohler2020equivariant} to the invariance group $\invgroup$.
\begin{definition}
    A diffeomorphism $T : \gC \to \gC$ is $\invgroup$-equivariant if for all $g \in \invgroup$, $T \circ g =g \circ T$ holds.
\end{definition}
\begin{definition}
    A diffeomorphism $T : \gC \to \gC$ is $\invgroup$-equivariant if for all $g \in \invgroup$, $T \circ g =g \circ T$ holds.
\end{definition}

\begin{lemma}\label{lemma:t_inv_equi}
    Let $g \in \invgroup$ and $T : \gC \to \gC$ be a diffeomorphism. $T$ is $\invgroup$-equivariant if and only if $T^{-1}$ is $\invgroup$-equivariant.
\end{lemma}
\begin{proof}
    Let $Y \in \gC$ and set $X = T(Y)$, then
    \begin{align*}
        T^{-1}(g(X)) = g(T^{-1}(X)) &\Longleftrightarrow T^{-1}\left(g(T(Y))\right) = g(Y) \\
        &\Longleftrightarrow g(T(Y)) = T(g(Y))\eqsp . 
    \end{align*}
\end{proof}

The following proposition generalizes \cite[Theorem 1]{kohler2020equivariant} to $\invgroup$ which is nonlinear.
\begin{proposition}\label{prop:eq_transport_map}
    Let $q : \gC \to \sR$ be a density of $\gC$ and let $T : \gC \to \gC$ be a diffeomorphism. If $q$ is $\invgroup$-invariant and $T$ is a $\invgroup$-equivariant, then the push-forward of $q$ through $T$ is also $\invgroup$-invariant.
\end{proposition}
\begin{proof}
    Let $C\in\gC$,  the push-forward of $q$ through $T$ writes as
    $$
        p(C)=q\left(T^{-1}(C)\right)\abs{J_{T^{-1}}(C)}\eqsp.
    $$
    Thus
    $$
        p\left(g(C)\right)=q\left(T^{-1}\left(g(C)\right)\right)\abs{J_{T^{-1}}\left(g(C)\right)}\eqsp .
    $$
    By~\Cref{lemma:t_inv_equi} and the invariance of $q$ we have
    \begin{align*}
        q\left(T^{-1}\left(g(C)\right)\right)&=q\left(g\left(T^{-1}\left(C\right)\right)\right)\\
        &=q\left(T^{-1}(C)\right).
    \end{align*}
    Moreover,
    \begin{align*}
        \abs{J_{T^{-1}}\left(g(C)\right)}&=\frac{\abs{J_{T^{-1}\circ g}(C)}}{\abs{J_g(C)}} & & \left(\text{Chain rule}\right)\\
        &=\abs{J_{T^{-1}\circ g}(C)}& & \left(\text{\Cref{lemma:jacobian_g}}\right)\\
        &=\abs{J_{g\circ T^{-1}}(C)}& & \left(\text{\Cref{lemma:t_inv_equi}}\right)\\
        &=\abs{J_g\left(T^{-1}(C)\right)}\abs{J_{T^{-1}}(C)}& & \left(\text{Chain rule}\right)\\
        &=\abs{J_{T^{-1}}(C)}& &\left(\text{\Cref{lemma:jacobian_g}}\right).
    \end{align*}
    Combining these results we get $p\left(g(C)\right)=p(C)$.
\end{proof}

\begin{definition}
    A velocity field $\hat v : [0,1] \times \gC \to \gT \gC$ is \emph{$\invgroup$-equivariant} if, for all $g \in \invgroup$, using the decomposition $g = f_L \circ h_{A,b}$ (with $A$ and $b$ defined in \Cref{lemma:decomp_g}),
    $$
        \hat v(t,g(C)) = A\, \hat v(t,C), \quad \forall C \in \gC\eqsp.
    $$
\end{definition}

\begin{proposition}
    \label{prop:equivariant_field_to_equivariant_map}
    If $v : [0,1] \times \gC \to \gT \gC$ is a Lipschitz-bounded $\invgroup$-equivariant velocity field, then the transport map induced by its ODE flow is $\invgroup$-equivariant.
\end{proposition}
\begin{proof}
    Let $g \in \invgroup$ be decomposed (as per \Cref{lemma:decomp_g}) as $g = f_L \circ h_{A,b}$. Let $Z \in \gC$. Consider $C(\cdot, Z)$ denotes the solution of the ODE with initial condition $Z \in \gC$. Let $\tilde{C}(\cdot, Z) = C(\cdot, g(Z))$. Using the chain rule and the equivariance of $v$, we have almost everywhere that
    $$
        \frac{\rmd \tilde{C}(t, Z)}{\rmd t} = A \frac{\rmd C(t, Z)}{\rmd t} = A v(t, C(t, Z)) = v(t, g(C(t, Z))) = v(t, \tilde{C}(t, Z))\eqsp.
    $$
    Moreover, $\tilde{C}(0, Z) = g(Z)$. Using the unicity of the ODE's solutions (due to the velocity field being Lipschitz-bounded), we get
    $$
        C(t, g(Z)) = g(C(t, Z)), \quad ~\text{almost everywhere, for all }~ t \in [0,1] \eqsp.
    $$
    Using this into the integral definition of the transport map, for all $t \in [0,1]$ we get
    \begin{align*}
        T(g(Z)) &= f_L\left(g(Z) + \int_0^t  v(u, C(t, g(Z))) \rmd u\right) \\
        &= f_L\left(g(Z) + \int_0^t  v(u, g(C(t, Z))) \rmd u\right) & & \left(\text{Using } C(\cdot, g(Z)) = g(C(\cdot, Z))\right) \\
        &= f_L\left(f_L(h_{A,b}(Z)) + A \int_0^t  v(u, C(t, Z)) \rmd u\right) & & \left(\text{Using the equivariance of } v\right) \\
        &= f_L\left(h_{A,b}(Z) + A \int_0^t  v(u, C(t, Z)) \rmd u\right) & & \left(\text{Using \Cref{lemma:torus:torus_decomp}}\right) \\
        &= f_L\left(h_{A,b}\left(Z + \int_0^t  v(u, C(t, Z)) \rmd u\right)\right) \\
        &= g(T(Z))\eqsp.
    \end{align*}
\end{proof}

\begin{proposition}\label{prop:egnn_is_equi}
    The velocity field in \Cref{eq:architecture} is $\invgroup$-equivariant.
\end{proposition}
\begin{proof}
    Let $t \in [0,1]$ and $C \in \gC$. We define
    \begin{align*}
        \psi_k(t, C) = (\psi^s_k(t, C), \psi^X_k(t, C)) = (s, X^k), ~~\Gamma_k(t, C) = H^k ~\text{ and }~ \Lambda_k(t, C) = M^k\eqsp.
    \end{align*}
    We start by showing by induction on $k \in \mathbb{N}$ that for all $g \in \invgroup$ and $i, j \in \iinter{1}{N}$, we have
    $$
        \psi_k(t, g(C)) = g(\psi_k(t, C)), ~ \left[\Gamma_k(t, g(C))\right]_{i} = \left[\Gamma_k(t, C)\right]_{\sigma(i)} ~\text{ and }~ \left[\Lambda_k(t, g(C))\right]_{i,j} = \left[\Lambda_k(t, C)\right]_{\sigma(i),\sigma(j)}\eqsp,
    $$
    where, building on \Cref{eq:decomp_g}, we decompose $g$ as
    $$
        g : \begin{pmatrix}
            s \\
            X
        \end{pmatrix} \mapsto \begin{pmatrix}
            g^s(C) \\
            g^X(C)
        \end{pmatrix} = \begin{pmatrix}
            A_{\sigma} s \\
            (A_{\sigma} \otimes I_d) h_{M, b, L}(X)
        \end{pmatrix}\eqsp,
    $$
    $$
        h_{M, b, L}(X) = \torusmod{(I_N \otimes M) X + (\mathbf{1}_N \otimes b)}\eqsp,
    $$
    where $h_{M, b, L}$ describes the group of invariances restricted to $\torus^N$ (denoted $\gG_{\torus^N}$), $\sigma \in S_N$ is a permutation, $A_{\sigma}$ is the associated permutation matrix of size $N \times N$, $M$ is a signed permutation matrix of size $d \times d$ and $b$ is a vector in $\sR^{d}$.

    At $k = 0$, we have $\psi_0(t, g(C)) = g(C) = g(\psi_0(t, C))$ and for all $i \in \iinter{1}{N}$
    $$
         \left[\Gamma_0(t, g(C))\right]_i = \left(t, \left[g^s(C)\right]_i\right) = \left(t, s_{\sigma(i)}\right) = \left[\Gamma_0(t, C)\right]_{\sigma(i)}\eqsp.
    $$
    Moreover, for all $i,j \in \iinter{1}{N}$
    $$
        \left[\Lambda_0(t, g(C))\right]_{i,j} = \hat{\phi}_e\left(\left[\Gamma_0(t, g(C))\right]_i, \left[\Gamma_0(t, g(C))\right]_j, \torusdist{\selectparticle{h_{M, b, L}(X)}{\sigma(i)}}{\selectparticle{h_{M, b, L}(X)}{\sigma(j)}}\right)\eqsp.
    $$
    Using the previous statement on $\Gamma_0$ as well as the invariance of $\mathrm{d}_{\torus}$ to $\gG_{\torus}$ (see \Cref{cor:torus:dist_is_invar}),
    \begin{align*}
        \left[\Lambda_0(t, g(C))\right]_{i,j} &= \hat{\phi}_e\left(\left[\Gamma_0(t, C)\right]_{\sigma(i)}, \left[\Gamma_0(t, C)\right]_{\sigma(j)}, \torusdist{\selectparticle{X}{\sigma(i)}}{\selectparticle{X}{\sigma(j)}}\right)\eqsp,\\
        &= \left[\Lambda_k(t, g(C))\right]_{\sigma(i),\sigma(j)}\eqsp.
    \end{align*}

    Let $k \in \mathbb{N}$ and assume that
    $$
        \psi_k(t, g(C)) = g(\psi_k(t, C)), ~ \left[\Gamma_k(t, g(C))\right]_{i} = \left[\Gamma_k(t, C)\right]_{\sigma(i)} ~\text{ and }~ \left[\Lambda_k(t, g(C))\right]_{i,j} = \left[\Lambda_k(t, C)\right]_{\sigma(i),\sigma(j)}\eqsp,
    $$
    Let $i, j \in \iinter{1}{N}$, using the recursion assumption and the permutation invariance of the sum, we have that
    \begin{align*}
        \left[\Gamma_{k+1}(t, g(C))\right]_i &= \hat{\phi}_h\left(\left[\Gamma_{k}(t, g(C))\right]_i, \sum_{i \neq j} \hat{\phi}_m\left(\left[\psi_k(t, g(C))\right]_{i,j}\right) \left[\psi_k(t, g(C))\right]_{i,j}\right)\eqsp, \\
        &= \hat{\phi}_h\left(\left[\Gamma_{k}(t, C)\right]_{\sigma(i)}, \sum_{i \neq j} \hat{\phi}_m\left(\left[\psi_k(t, C)\right]_{\sigma(i),\sigma(j)}\right) \left[\psi_k(t, C)\right]_{\sigma(i),\sigma(j)}\right)\eqsp, \\
        &= \hat{\phi}_h\left(\left[\Gamma_{k}(t, C)\right]_{\sigma(i)}, \sum_{i \neq j} \hat{\phi}_m\left(\left[\psi_k(t, C)\right]_{\sigma(i),j}\right) \left[\psi_k(t, C)\right]_{\sigma(i),j}\right)\eqsp, \\
        &= \left[\Gamma_{k+1}(t, C)\right]_{\sigma(i)}\eqsp.
    \end{align*}
    Similarly, using the same argument as in $k = 0$, it is easy to show that
    $$
        [\Lambda_{k+1}(t, g(C))]_{i,j} = [\Lambda_k(t, C)]_{\sigma(i),\sigma(j)}\eqsp.
    $$
    Additionally, we have
    $$
        \psi_{k+1}^s(t, g(C)) = g^s(C) = g^s(\psi_{k}^s(t, C))\eqsp.
    $$
    Moreover, using the definition and the previous statement on $\Lambda_k$
    \begin{align*}
        \selectparticle{\psi^X_{k+1}(t, g(C))}{i} &= \exp_{\selectparticle{\psi^X_k(t, g(C))}{i}} \left(\sum_{i \neq j} \frac{\log_{\selectparticle{\psi^X_k(t, g(C))}{j}} \selectparticle{\psi^X_k(t, g(C))}{i}}{\torusdist{\selectparticle{g^X(C)}{i}}{\selectparticle{g^X(C)}{j}} + 1} \phi_d([\Lambda_k(t, g(C))]_{i,j})\right)\eqsp, \\
        &= \exp_{\selectparticle{\psi^X_k(t, g(C))}{i}} \left(\sum_{i \neq j} \frac{\log_{\selectparticle{\psi^X_k(t, g(C))}{j}} \selectparticle{\psi^X_k(t, g(C))}{i}}{\torusdist{\selectparticle{g^X(C)}{i}}{\selectparticle{g^X(C)}{j}} + 1} \phi_d([\Lambda_k(t, C)]_{\sigma(i),\sigma(j)})\right)\eqsp.
    \end{align*}
    Using the recursion assumption together with \Cref{cor:torus:log_is_log} we have
    \begin{align*}
        \log_{\selectparticle{\psi^X_k(t, g(C))}{j}} \selectparticle{\psi^X_k(t, g(C))}{i} &= \log_{\selectparticle{g^X(\psi_k(t, C))}{j}} \selectparticle{g^X(\psi_k(t, C))}{i}\eqsp, \\
        &= \log_{\selectparticle{h_{M,b,L}(\psi^X_k(t, C))}{\sigma(j)}} \selectparticle{h_{M,b,L}(\psi^X_k(t, C))}{\sigma(i)}\eqsp, \\
        &= M \log_{\selectparticle{\psi^X_k(t, C)}{\sigma(j)}} \selectparticle{\psi^X_k(t, C)}{\sigma(j)}\eqsp.
    \end{align*}
    Using \Cref{prop:torus:M_and_exp} and the recursion assumption again, we get that
    $$
        \selectparticle{\psi^X_{k+1}(t, g(C))}{i} = \torusmod{M \exp_{\selectparticle{\psi^X_k(t, C)}{\sigma(i)}} \sum_{i \neq j} \frac{\log_{\selectparticle{\psi^X_k(t, C)}{\sigma(j)}} \selectparticle{\psi^X_k(t, C)}{\sigma(i)}}{\torusdist{\selectparticle{g^X(C)}{i}}{\selectparticle{g^X(C)}{j}} + 1} \phi_d([\Lambda_k(t, C)]_{\sigma(i),\sigma(j)}) + b}\eqsp.
    $$
    Together with the invariance of $\mathrm{d}_{\torus}$ and the permutation invariance of the sum, it leads to
    \begin{align*}
        \selectparticle{\psi^X_{k+1}(t, g(C))}{i} &= \torusmod{M \exp_{\selectparticle{\psi^X_k(t, C)}{\sigma(i)}} \sum_{i \neq j} \frac{\log_{\selectparticle{\psi^X_k(t, C)}{j}} \selectparticle{\psi^X_k(t, C)}{\sigma(i)}}{\torusdist{\selectparticle{X}{\sigma(i)}}{\selectparticle{X}{j}} + 1} \phi_d([\Lambda_k(t, C)]_{\sigma(i),j}) + b}\eqsp,\\
        &= \selectparticle{\left[\torusmod{M \selectparticle{\psi^X_{k+1}(t, g(C))}{i} + b}\right]}{\sigma(i)}\eqsp, \\
        &= \selectparticle{g^X(\psi^X_{k+1}(t, C))}{i}\eqsp,
    \end{align*}
    which concludes the recursive proof. We get the overall proof using \Cref{cor:torus:log_is_log} as
    \begin{align*}
         \hat{v}(t, g(C)) &= \begin{pmatrix}
            \mathbf{0}_{N \times d_s} \\
            \left\{\log_{\selectparticle{g^X(X)}{i}} \selectparticle{\psi^X_{K}(t, g(C))}{i}\right\}_{i=1}^N
        \end{pmatrix}\eqsp, \\
        &= \begin{pmatrix}
            \mathbf{0}_{N \times d_s} \\
            \left\{\log_{\selectparticle{g^X(X)}{i}} \selectparticle{g^X(\psi^X_{K}(t, C))}{i}\right\}_{i=1}^N
        \end{pmatrix}\eqsp, & \left(\text{Recursion}\right)\\
        &= \begin{pmatrix}
            \mathbf{0}_{N \times d_s} \\
            \left\{M \log_{\selectparticle{X}{\sigma(i)}} \selectparticle{\psi^X_{K}(t, C)}{\sigma(i)}\right\}_{i=1}^N
        \end{pmatrix}\eqsp, & \left(\text{\Cref{cor:torus:log_is_log}}\right)\\
        &= A \hat{v}(t, C)
    \end{align*}
\end{proof}

\section{Numerical details} 

\label{app:impl_details}

In this section, we report the implementation, architectural, and training details for all numerical simulations on the physical systems defined in \Cref{sec:methods}. For the Equivariant Riemannian Stochastic Interpolant (ERSI), the velocity field is parameterized via our Equivariant Graph Neural Network (EGNN) architecture (\Cref{eq:architecture}). In the base $K=3, \text{HF}=32$ network (containing $\approx 22\text{k}$ parameters), particle species are embedded via one-hot encoding, $\hat{\phi}_e$ and $\hat{\phi}_h$ are 3-layer MLPs of width 32, and $\hat{\phi}_m$ is a 2-layer MLP of width 32.
For scaled ERSI architectures ($4 \mid 64$, $5 \mid 128$, and $10 \mid 128$), the MLP hidden widths are set equal to $\text{HF}$, and the number of message-passing layers is set to $K$.
For the EFM baseline, we adopt the architecture of \cite{satorras2021en_gnn} with standard Euclidean operators. 
For the non-equivariant RSI baseline, particle configurations are processed directly by a 3-layer MLP with 64 hidden units per layer, projecting coordinates onto the torus following \cite[Eq.~26]{chen2024flow}; testing explicit data augmentation via random group actions of $\invgroup$ during training yielded no noticeable performance recovery. 
All models were trained on $10^5$ equilibrium configurations using the Adam optimizer.
Hyperparameters—including learning rates, batch sizes, and gradient clipping settings—were selected via validation loss and are summarized in \Cref{tab:hyperparams}.
Generative sampling and exact likelihood evaluations (\Cref{eq:ode_likelihood}) are performed using the adaptive \texttt{dopri5} Runge–Kutta solver from \texttt{torchdiffeq} \cite{torchdiffeq} with absolute and relative integration tolerances set to $10^{-5}$.

\begin{table}[h!]
	\centering
	\small
	\caption{Hyperparameter settings, network architectures $(K \mid \text{HF})$, batch sizes, training steps, learning rates, and gradient clipping thresholds across all studied glass models, state points, and baseline methods.}
	\label{tab:hyperparams}
	\begin{tabular}{lcclcccccc}
		\toprule
		System & $N$ & $T$ & Model & Architecture $(K \mid \text{HF})$ & Batch Size & Training Steps & Learning Rate & Gradient Clipping \\
		\midrule
		\multirow{6}{*}{\shortstack[c]{Binary\\Mixture}} 
               & 10 & 0.10 & ERSI & EGNN $3 \mid 32$  & 2048 & $5  \times 10^4$ & $1\times10^{-4}$ & 2.0 \\
		       & 10 & 0.10 & EFM  & EGNN $3 \mid 32$  & 2048 & $5  \times 10^4$ & $5\times10^{-3}$ & 2.0 \\
		       & 10 & 0.10 & RSI  & MLP          & 2048 & $5  \times 10^4$ & $5\times10^{-4}$ & 2.0 \\
		\cmidrule(lr){2-9}
               & 44 & 0.10 & ERSI & EGNN $3 \mid 32$  & 1024 & $10^5$ & $5\times10^{-3}$ & 2.0 \\
		       & 44 & 0.10 & EFM  & EGNN $3 \mid 32$  & 1024 & $10^5$ & $1\times10^{-3}$ & 2.0 \\
		       & 44 & 0.10 & RSI  & MLP          & 1024 & $10^5$ & $5\times10^{-5}$ & 2.0 \\
		\midrule
		\multirow{4}{*}{\shortstack[c]{Ternary\\Mixture}} 
		       & 44 & 1.00 & ERSI & EGNN $3 \mid 32$  & 1024 & $5  \times 10^4$ & $5\times10^{-3}$ & None \\
		       & 44 & 1.00 & ERSI & EGNN $4 \mid 64$  & 1024 & $5  \times 10^4$ & $5\times10^{-3}$ & None \\
		       & 44 & 1.00 & ERSI & EGNN $5 \mid 128$ & 1024 & $5  \times 10^4$ & $1\times10^{-3}$ & None \\
		\cmidrule(lr){2-9}
		       & 44 & 0.32 & ERSI & EGNN $10 \mid 128$ & 512  & $2.5 \times 10^4$ & $1\times10^{-4}$ & None \\
		\bottomrule
	\end{tabular}
\end{table}

\end{document}